\documentclass{article}

\usepackage[main, final]{neurips_2026}
\usepackage{lineno}
\usepackage{natbib}

\usepackage[hyperindex,breaklinks,colorlinks,citecolor=blue]{hyperref} % allow link break (for arxiv)
\usepackage[utf8]{inputenc} % allow utf-8 input
\usepackage[T1]{fontenc}    % use 8-bit T1 fonts
\usepackage{url}            % simple URL typesetting
\usepackage{booktabs}       % professional-quality tables
\usepackage{amsfonts}       % blackboard math symbols
\usepackage{nicefrac}       % compact symbols for 1/2, etc.
\usepackage{microtype}      % microtypography
\usepackage{xcolor}         % colors
\usepackage[flushleft]{threeparttable}
\usepackage{tabularray}
\usepackage{thmtools}
\usepackage{mathtools}
\usepackage{extarrows}
\usepackage{colortbl}
\usepackage[table]{xcolor}
\usepackage{nccmath}
\usepackage{subcaption}
\usepackage{framed}
\colorlet{shadecolor}{gray!35}

\usepackage{amsthm}

\theoremstyle{definition}
\newtheorem{condition}{Condition}

\usepackage{url}
\usepackage{float}
\usepackage{subcaption}
\usepackage{graphicx}
\usepackage{enumitem}
\usepackage{multirow}
\usepackage{xspace}
\usepackage{enumitem}
\usepackage[font=small]{caption}
\usepackage[toc, page, header]{appendix}
\usepackage{footnote}
\usepackage{amsmath,amsthm,amssymb}
\newtheorem{theorem}{Theorem}[section]

\newtheorem{corollary}[theorem]{Corollary}

\newtheorem{remark}[theorem]{Remark}

\newtheorem{claim}[theorem]{Claim}

\providecommand{\uu}{\mathbf{u}}

\providecommand{\zz}{\mathbf{z}}
\providecommand{\unit}{\mathbf{1}}

\providecommand{\mM}{\mathbf{M}}

\providecommand{\mX}{\mathbf{X}}

\providecommand{\mZ}{\mathbf{Z}}

\providecommand{\cL}{\mathcal{L}}

\newenvironment{talign*}
{\csname align*\endcsname}
{\endalign}

\newcommand{\miss}{\mathrm{miss}}
\newcommand{\obs}{\mathrm{obs}}

\newcommand{\indep}{\perp \!\!\! \perp}
\newcommand{\xx}{\boldsymbol{x}}
\newcommand{\mm}{\boldsymbol{m}}
\newcommand{\mP}{\mathbb{P}}
\usepackage{tikz}

\definecolor{blue}{rgb}{0.0, 0.0, 0.7}
\usepackage{algorithm}
\usepackage{algpseudocode}
\usepackage{colortbl}
\usepackage{xcolor}

\usepackage{wrapfig} 
\usepackage{amsmath}
\allowdisplaybreaks
\author{%
\begin{tabular}{cccc}
{Yue Liu } & {Ben Liang} & {Ali Tizghadam} & {Ilijc Albanese} \\
 \textnormal{University of Toronto}    & \textnormal{University of Toronto}& \textnormal{TELUS}&\textnormal{TELUS} \\
\end{tabular}
}

\title{Integrated Imputation-Classification \\for Supervised Learning with Missing Data}

\begin{document}

\maketitle

\begin{abstract}
 %Existing solutions typically treat imputation and classification as separate steps, which makes the classifier inherently biased toward one specific imputation rather than being optimal across the full distribution of plausible completions---a consequence of the mismatch between imputation and the classification objective. 
We study supervised classification problems with missing feature values. Existing approaches often decouple imputation from classification, producing imputations that may be plausible but uninformative for prediction. Instead, we propose the Integrated Imputation and Classification Network (IICN), which jointly trains an imputer and an ${(n{+}1)}$-class discriminator adversarially with a single class-supervised objective, where the discriminator learns to distinguish among the $n$ true classes and an additional ``imputed'' class. We prove that it admits a global optimum where the imputer and discriminator together implement marginalization over missing coordinates and yield a Bayes-optimal classifier. We evaluate IICN on FashionMNIST, CIFAR-10, and tabular datasets with naturally occurring missingness. IICN outperforms classical impute-then-classify pipelines and recent generative baselines, showing strong robustness and accuracy in challenging settings.
\end{abstract}

\section{Introduction}

Missing data are ubiquitous in real-world scenarios, due to for example a non-response in surveys or questionnaires, sensor failures, system errors, or transmission loss. The existence of such missingness brings challenges to downstream classification tasks. While typical supervised learning assumes fully observed inputs, learning under missing features is both practically necessary and theoretically nontrivial.

We formally define the classification task under incomplete features as follows. Let $(\xx, y)$ denote the original complete labelled datapoint, where $\xx\in\mathbb{R}^d$ is the feature vector and $y$ is the corresponding label. Some entries of $\xx$ are obscured, resulting in an incomplete datapoint $\tilde \xx$. A binary missing mask vector $\mm\in\{0,1\}^{d}$ is directly inferred from $\tilde\xx$,  i.e., $\mm_j = 1$ if $\xx_j$ is observed and $0$ otherwise. Then in practice we have access only to the observation $(\tilde\xx,y)$, where $\tilde\xx = \xx\odot\mm+\texttt{nan}\odot(\unit-\mm),$ and $\mathbf{1}=[1,..,1]^T$ is the all-ones vector and $\odot$ is the Hadamard product.\footnote{We define that $\texttt{nan}\cdot 0=0$ and $\texttt{nan}\cdot x_j=\texttt{nan}$.} Let $\xx_{\obs} = \{\xx_j:\mm_j = 1\}$ and $\xx_{\miss} = \{\xx_j:\mm_j = 0\}$ denote the observed and missing features. Given a training dataset that contains both complete and incomplete datapoints, our goal is to train a model that accurately classifies both complete and incomplete test data.
\vspace{-2mm}
\paragraph{Limitations of impute-then-classify}A widely used strategy is the impute-then-classify pipeline, which first replaces missing features $\xx_{\miss}$ by $\hat\xx_{\miss}$ and then feeds the completed input to a classifier. Despite its simplicity, this paradigm has fundamental limitations. (i) In high-dimensional settings, the combinatorial growth of missingness patterns makes accurate imputation statistically and computationally challenging. (ii) Even when learnable imputers can produce high-quality imputations, e.g., semantically plausible reconstructions~\citep{dcgan, yeh2017semantic, pathak2016context} or statistically well-calibrated estimates~\citep{gain, ipsen2021not}, such improvements in imputation quality often lead to only marginal downstream classification gains~\citep{paterakis2024we, morvan2025imputation} and may induce ambiguity or systematic bias~\citep{smieja2018processing, josse2024consistency}, because the imputation objective is not necessarily aligned with the downstream classification goal. (iii) The staged separation between imputation and prediction prevents the model from directly optimizing the classification objective, making it difficult to characterize the optimality of the resulting decision rule. %(iv) Moreover, this decoupling can amplify errors under shifts or corrupted observations, as imputation mistakes directly propagate to the final decision~\citep{MANGUSSI2025104574}.
\vspace{-2mm}
\paragraph{Prediction-aligned objective under missingness} The above limitations suggest that imputations should be evaluated and learned based on their impact on classification. Formally, the goal is to find a predictor that minimizes the expected classification risk under missingness: \nopagebreak
{
\setlength{\abovedisplayskip}{2pt}
\setlength{\belowdisplayskip}{2pt}
\setlength{\abovedisplayshortskip}{2pt}
\setlength{\belowdisplayshortskip}{2pt}
\textcolor{black}{
\begin{equation}
\label{eq:pred_obj}
\min_{f}\ \mathbb{E}\Big[
\ell\big(y, f(\xx_{\obs}, \mm)\big)
\Big],
\end{equation}}}where $\ell(\cdot,\cdot)$ denotes a classification loss and $f(\cdot,\cdot)$ denotes the predictor on missing data. %$\underset{\mathrm{miss}}{\hat{\xx}^{\star}}=\arg\underset{\hat\xx_{\miss}\in\mathbb{R}^{\mid \miss}\mid }{\min}\ l(y,\hat y_1)$ where $l$ is the loss to measure the classification performance and \begin{equation}\nonumber\hat y_1 = \arg\max P(y\mid \xx_{\obs},\hat\xx^{\star}_{\miss}).\end{equation}
 Recent studies show that such integration of imputation and classification can achieve Bayes-optimal prediction, even with a simple imputer, provided the classifier is sufficiently expressive~\citep{bayes_optimal, josse2024consistency}.

Several approaches have been proposed to address the predictive objective in Eq.~\eqref{eq:pred_obj}:

\textit{(i) Search-based methods.} When the missing patterns or the candidate set of  imputations are small, such as in low-dimensional or discrete settings~\citep{liu2022greedy,lee2022adaptive}, one can perform explicit search to find the combination of imputer and predictor that minimizes the classification loss. However, in general the number of missingness patterns grows exponentially with the input dimension, quickly rendering such search-based procedures computationally infeasible.

\textit{(ii) Marginalization-based methods.} An alternative line of work directly targets the Bayes-optimal predictor, which takes the form of marginalization over missing values: {
\setlength{\abovedisplayskip}{4pt}
\setlength{\belowdisplayskip}{4pt}
\setlength{\abovedisplayshortskip}{4pt}
\setlength{\belowdisplayshortskip}{4pt}
\[
\mP(y\mid\xx_\obs)= \int \mP(y \mid \xx_{\obs}, \xx_{\miss})\,
p(\xx_{\miss} \mid \xx_{\obs})\, \mathrm{d}\xx_{\miss}.
\]
}Early works~~\citep{ahmad1992some,tresp1994efficient,gmm} model $p(\xx_\miss\mid \xx_\obs)$ via parametric densities (e.g., GMMs) with generalized linear classifiers to enable closed-form marginalization, but they are restricted to low-dimensional or simple distributions. NeuMiss~\citep{neumiss} approximates the Bayes predictor through a Neumann series under linear-Gaussian assumptions, which limits its applicability to highly nonlinear tasks. 

\textit{(iii) Joint optimization methods.} The approach most relevant to ours is to jointly train an imputer and a classifier on the prediction loss. GACN~\citep{gacn} combines GAIN~\citep{gain} with an external classifier, propagating the classification loss on the imputed data from the classifier back through the imputer to align imputations with the predictive objective. However, it lacks theoretical insight on how the two modules interact and to which predictive rule the joint optimization converges. LAC~\citep{sugiyama2025discrete, sugiyama2025unified} minimizes a weighted loss combining reconstruction error and classification loss through alternating optimization between the imputer and the classifier. While LAC enjoys an asymptotic risk bound, it relies on point estimation rather than marginalization and does not provide any provable property about the decision rule's calibration. SupMIWAE~\citep{SupMIWAE} jointly trains a variational imputation 
posterior and a classifier, approximating the predictive distribution via importance-weighted samples drawn from the posterior. While SupMIWAE accounts for imputation uncertainty through marginalization, its training 
objective is a variational lower bound on the log-likelihood, and 
performance depends on the fidelity of the generative model and 
the expressiveness of the variational family.%SupMIWAE~\citep{SupMIWAE} approximates the marginalization by drawing imputations from a variational posterior and combining them via importance weighting. While scalable, its training objective is just a variational lower bound of the likelihood and performance depends on the fildelity of the generative model and the expressiveness of the variational family.  

In this work, we propose the \textbf{Integrated Imputation and Classification Network (IICN)}, a unified end-to-end framework that addresses the limitations above through a class-aware adversarial formulation. By coupling imputation and classification within a single $(n{+}1)$-class adversarial objective, IICN is likelihood-free and expressive enough to capture the conditional distribution without parametric or variational restriction. The composition of the imputer and discriminator performs marginalization over missing coordinates and realizes a classifier that is provably Bayes-optimal at the adversarial equilibrium. Further comparison of IICN with related work is given in Appendix~\ref{sec:related}.

Our main contributions are summarized as follows:
\vspace{-1mm}
\begin{itemize}
\item \textbf{An $(n{+}1)$-class supervised adversarial framework with Bayes-optimality.}  We adopt the $(n{+}1)$-class semi-supervised 
Generative Adversarial Network (GAN) formulation to missing-data 
classification, where an auxiliary class encodes data completeness. We prove that, under the model of missing-completely-at-random (MCAR) assumption, the composition of the optimal generator $G^\star$ and discriminator $D^\star$ implements the Bayes-optimal marginalization, %$$\mP(Y=i \mid \xx_\obs) = \int \mP(Y=i \mid \mX = (\xx_\obs, \xx_\miss))\, p(\xx_\miss \mid \xx_\obs)\, \mathrm{d}\xx_\miss,$$
with $G^\star$ providing the imputation distribution and $D^\star$ the class posterior, without training a separate classifier.% (Thm.~\ref{thm:classifier}).

\item \textbf{A stable training scheme.} Vanilla adversarial training 
under missingness can degenerate, since the discriminator may learn to detect 
imputed inputs from mask-derived artifacts rather than semantic content. 
We propose an alternating update scheme that switches the discriminator between adversarial detection and supervised classification, besides the naive training of generator and discriminator updates. It prevents the discriminator from easily exploiting non-semantic artifacts and missingness cues and stabilizes training.

\item \textbf{Strong empirical performance, particularly under severe 
missingness.} On FashionMNIST, CIFAR-10, and tabular datasets with real missingness, IICN outperforms classical impute-then-classify pipelines and 
recent generative baselines. Empirically, IICN remains robust under severe missingness, demonstrating its empirical effectiveness.
\end{itemize}

\section{IICN: Integrated Imputation and Classification Network}
%In order to achieve the Bayes-optimal predictor, we follow an Impute-then-Classify procedure: the data is first imputed via a function $G$, and then classification is performed on the imputed input using a classifier $D$. We design IICN to jointly learn the best pair of imputer and classifier such that the imputation process aligns with the classification objective.
We now introduce IICN, a unified framework that jointly learns an imputer $G$ and a classifier $D$ via a single $(n{+}1)$-class adversarial objective. We present the framework in Section~\ref{sec:framework} and establish its theoretical guarantees in Section~\ref{sec:theory}.

\subsection{Multi-class Supervised Adversarial Training Framework for IICN} \label{sec:framework}

\begin{wrapfigure}{r}{0.5\textwidth}
    \vspace{-3.5mm}
	\centering
  	\includegraphics[width=0.45\textwidth]{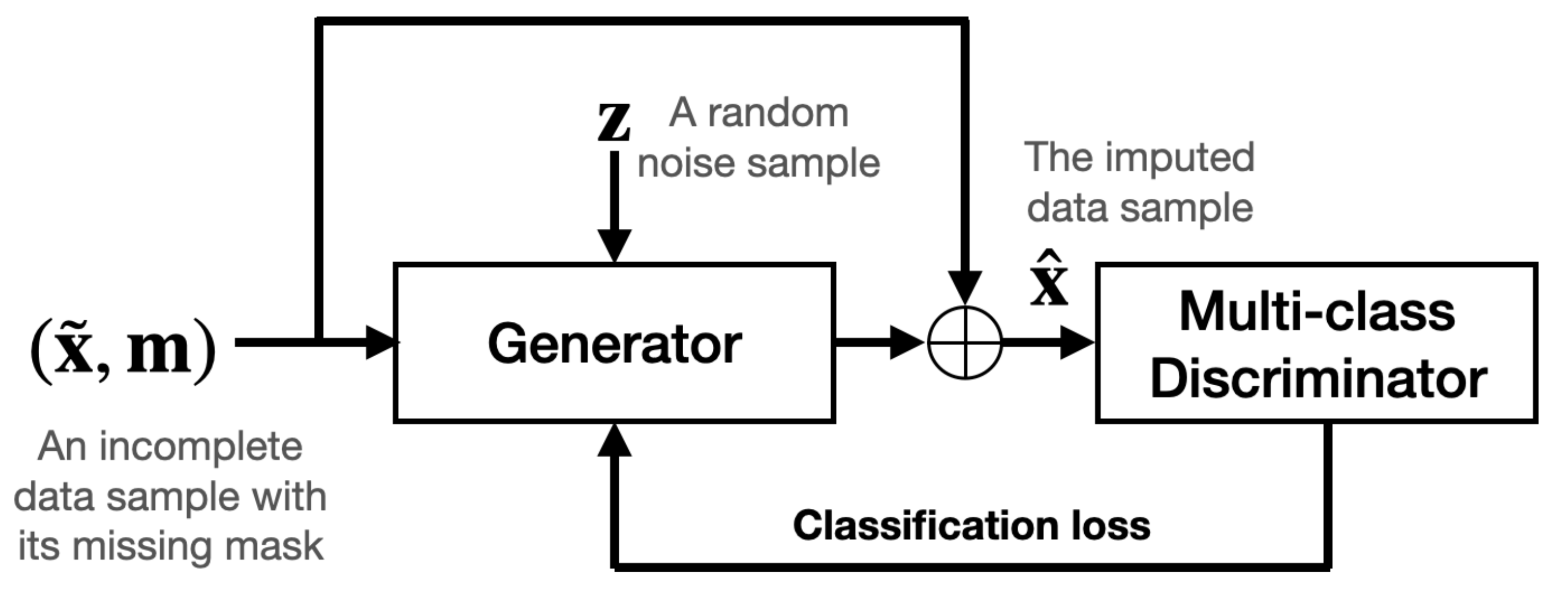}
	\caption{Schematic of the IICN network. The completed sample $\hat{\xx}$ is made up of the output of the generator and the observed parts of an incomplete sample $\tilde\xx$. The completed sample $\hat\xx$ is passed to a multi-class discriminator. The classification loss is applied to update both the generator and the discriminator.}
    \vspace{-12pt}
	\label{fig:arch}
\end{wrapfigure}

We introduce a multi-class supervised adversarial training framework designed to jointly optimize imputation and classification under missingness. IICN uses a multi-class supervised discriminator $D$ that distinguishes among $n{+}1$ classes. It replaces the binary real-vs-fake classification of standard GANs~(and the coordinate-wise observed-vs-imputed decision of GAIN~\citep{gain}) with a class-level decision that directly aligns the imputer with the downstream classification objective. The generator $G$ fills in missing features conditioned on the observed ones and a latent code, and it is trained adversarially to fool $D$ into predicting one of the $n$ semantic classes. The overall architecture is illustrated in Fig.~\ref{fig:arch}. 
%Unlike traditional GANs that distinguish real from fake samples via binary discrimination, our discriminator is trained to classify each input into one of the original $n$ classes or an additional \textit{imputed} class. Through adversarial training, $D$ provides class-level feedback on imputed inputs, guiding $G$ to generate completions that are not only plausible but also discriminative for downstream prediction.

\subsubsection{Generator}
The generator $G$ takes as input a partially observed datapoint $\tilde\mX \in (\mathbb{R}\cup\{\mathrm{nan}\})^d$, a binary mask $\mM \in \{0,1\}^d$, and a latent variable $\mZ \sim P_Z$ sampled from a prior distribution with support $\mathcal{Z}$. The generator outputs a full-length vector via a mapping 
$G: (\mathbb{R}\cup\{\mathrm{nan}\})^d \times \{0,1\}^d \times \mathcal{Z} \rightarrow \mathbb{R}^d$, which is then merged with the observed entries to form the completed sample:\footnote{For notational clarity, we follow the principle that uppercase letters (e.g., $\mX,Y,\mM$) denote random variables and lowercase letters (e.g., $\xx,y,\mm$) denote their values. }
$$\hat\mX = \tilde\mX \odot \mM + G(\tilde\mX, \mM, \mZ) \odot (1 - \mM).$$
Note that on the event $\{\mM=\unit\}$, all dimensions of the datapoint are observed so $\hat\mX=\tilde\mX=\mX$.

\subsubsection{Multi-class Discriminator}
In addition to distinguishing the original $n$ classes, we introduce an auxiliary $(n{+}1)$-th class labelled as \textit{imputed}. Samples with complete data (i.e., $\mM = \mathbf{1}$) retain their original class labels, while imputed samples (i.e., those with $\mM \neq \mathbf{1}$) are assigned to the imputed class. This extension enables the discriminator to jointly capture both semantic class identity and data completeness.

Formally, the discriminator is defined as $D: \mathbb{R}^d \rightarrow \mathcal{P}^n$, where $\mathcal{P}^n := \left\{ \mathbf{p} \in (0,1)^n \,\middle| \, \sum_{i=1}^{n} p_i < 1 \right\}.$ Given an input $\hat\xx \in \mathbb{R}^d$, the output $D(\hat\xx)$ is a vector $[D_1(\hat\xx), \dots, D_n(\hat\xx)]$ representing real-class probabilities, while the residual mass $1 - \sum_{i=1}^{n} D_i(\hat\xx)$ corresponds to the probability that the input is classified as an imputed (i.e., originally incomplete) sample.

\subsubsection{Supervised Adversarial Loss}
Let $R:=\mathbb{I}\{\mM=\mathbf{1}\}\in\{0,1\}$ denote whether a sample is fully observed~($R=1$) or incomplete~($R=0$). We define the adversarial objective $V(G, D)$ based on a supervised classification loss that augments the original $n$ classes with an auxiliary $(n{+}1)$-th label for imputed inputs. The discriminator learns to classify complete samples correctly while detecting imputed ones; the generator is trained to fool the discriminator by producing imputations confidently classified as one of the $n$ original classes.

This adversarial setup naturally forms a min-max objective: 
 {
\setlength{\abovedisplayskip}{2pt}
\setlength{\belowdisplayskip}{2pt}
\setlength{\abovedisplayshortskip}{2pt}
\setlength{\belowdisplayshortskip}{2pt}
\begin{equation}
\allowdisplaybreaks
\begin{aligned}
\underset{G}{\mathrm{min}}\ \underset{D_1,\dots,D_n}{\mathrm{max}}\ V(G, D)
&:= \mathbb{P}(R=1)\sum_{i=1}^n \mathbb{P}(Y=i)
\mathbf{E}_{\hat{\mX}\mid Y=i, R = 1}\log D_i(\hat{\mX}) \\
&\quad + \mathbb{P}(R=0)
\mathbf{E}_{\hat{\mX}\mid R=0}
\log\Big(1-\sum_{j=1}^n D_j(\hat{\mX})\Big).
\end{aligned}
\label{eq:loss}
\end{equation}}This encourages the generator to produce task-aligned imputations that are indistinguishable from real samples with respect to the downstream \textit{classification} objective.
 \vspace{-1mm}
\subsection{Optimality Analysis}\label{sec:theory}
We characterize the equilibrium of the min-max objective in Eq.~\eqref{eq:loss} under a non-parametric setting and show that there exists an optimal discriminator-generator pair $(G^\star, D^\star)$ that composes a Bayes-optimal classifier under missingness~(Thm.~\ref{thm:classifier}). Throughout this analysis, we assume that feature values are missing completely at random~(MCAR)~\citep{bayes_optimal, neumiss, 
%supmiwae, 
gain%seaman2013meant
}. We first derive the optimal discriminator for a fix generator~(Section~\ref{subsec:opt_d}), then show that minimizing over generator yields a distribution alignment property~(Section~\ref{subsec:opt_g}). Finally, we combine these results to establish Bayes optimality~(Section~\ref{subsec:classifier}).
 %\begin{assumption}[Independence]
%\label{ass:joint}
%The missingness mechanism is independent of the underlying data.
%Formally, the joint random variable $(Y, \mX)$ is independent of the missing pattern $\mM$.
%\end{assumption}

\subsubsection{Optimal Discriminator}\label{subsec:opt_d}
For any fixed generator $G$, the optimal discriminator is defined as $D^\star := \arg\max_{D} V(G, D),$ where the objective $V(G,D)$ takes the form of a cross-entropy loss over an extended label space: the $n$ class labels for complete data and an additional $(n{+}1)$-th label for imputed samples. 

This is maximized when the predicted probabilities match the true posterior:
\begin{theorem}\label{thm:ds}The unique maximizer is given by 
\begin{equation}\nonumber
		D_i^{\star}(\hat\xx) %= \frac{p(y=i, \hat\xx, \mm= \mathbf{1})}{\mZ} 	\\
		 = \mP(Y=i, R=1\mid \hat\mX=\hat\xx) \quad \text{for all } \hat\xx \in \mathbb{R}^d.	\end{equation} 
\end{theorem}
The proof of this theorem is given in Appendix~\ref{sec:opt_d}.

This shows that $D_i^\star(\hat\xx)$ represents the joint probability that the input is fully observed and belongs to class $i$. Further, normalizing over the $n$ real classes yields
\begin{equation}\label{eq:frac_d}
\frac{D_i^\star(\hat\xx)}{\sum_j D_j^\star(\hat\xx)} 
= \mP(Y=i \mid \hat\mX = \hat\xx, R=1).
\end{equation}
Conditioning on $R=1$, we know that $\hat\mX=\mX$ since all dimensions are observed, hence the event $\{\hat\mX=\hat\xx, R=1\}$ and $\{\mX=\hat\xx, R=1\}$ coincide, giving  $$\mP(Y=i \mid \hat\mX = \hat\xx, R=1) = \mP(Y=i \mid \mX = \hat\xx, R=1).$$Finally, under MCAR assumption implies $(Y, \mX) \indep R$, we have $\mP(Y=i \mid \mX = \hat\xx, R=1) = \mP(Y=i\mid \mX=\hat\xx)$. Therefore, we can rewrite the Eq.~\eqref{eq:frac_d} as\begin{equation}\label{eq:normalized}\frac{D_i^\star(\hat\xx)}{\sum_j D_j^\star(\hat\xx)}  = \mP(Y=i\mid \mX = \hat\xx)\quad\text{for all }\hat\xx\in\mathbb{R}^d.\end{equation}Note that the right-hand side is the class posterior on the complete data evaluated at the value $\hat\xx$.

\subsubsection{Optimal Generator}\label{subsec:opt_g}
Plugging in the optimal discriminator $D^\star$ from Thm.~\ref{thm:ds} back into the min-max objective in Eq.~\eqref{eq:loss}, we obtain the $V(G, D^\star) $.

It attains its minimum at $G^\star:=\arg\min_G V(G, D^\star)$. We define the imputed random variable under the optimal generator $G^\star$ as $\hat\mX^\star := \tilde\mX \odot \mM + G^\star(\tilde\mX, \mM, \mZ) \odot (1 - \mM).$
The following theorem shows that at the optimum, the completed random 
variable $\hat\mX^\star$ is independent of the missingness indicator $R$.
\begin{theorem}[Independence]
\label{thm:generator}
 \textcolor{black}{Under MCAR,} the global minimum of $V(G, D^\star)$ is achieved if and only if \footnote{The existence of the minimizer relies on the fact that 
$\mP(R=1)\neq 0$, i.e., there is a non-zero probability of observing fully complete data points.}
$$p_{\hat\mX^\star}(\uu \mid R=1) = p_{\hat\mX^\star}(\uu \mid R=0), \quad \text{for all } \uu \in \mathbb{R}^d.$$
\end{theorem}
Equivalently, Thm.~\ref{thm:generator} implies $\hat\mX^\star \perp\!\!\!\perp R$, and hence $p_{\hat\mX^\star}(\uu\mid R=1)=p_{\hat\mX^\star}(\uu)$. See the proof in Appendix~\ref{sec:opt_g}.

Combined with the fact that $\hat\mX^\star=\mX$ when $R=1$, we have \begin{equation}\nonumber p_{\hat\mX^\star}(\uu \mid R = 1)=p_{\mX}(\uu \mid R = 1)\overset{\text{MCAR}}{=}p_{\mX}(\uu),\quad\text{for all } \uu \in \mathbb{R}^d.\end{equation}
\begin{corollary}[Distribution Alignment]
\label{cor:match}
$\hat\mX^\star \stackrel{d}{=} \mX,$
i.e., the imputation produced by $G^\star$ has the same 
distribution as the complete data.
\end{corollary}

As a direct consequence of Thm.~\ref{thm:generator}, the total probability mass assigned to the original $n$ classes becomes independent of the input, as stated below.
\begin{remark}\label{Claim:sum_ds}
Under the optimal generator $G^\star$ and discriminator $D^\star$, for 
any $\hat\xx^\star \in \mathbb{R}^d$,
$$
\sum_{i=1}^n D_i^\star(\hat\xx^\star) \xlongequal{\text{Thm.~\ref{thm:ds}}}\sum_{i=1}^n \mP(Y=i, R=1\mid \hat\mX^\star=\hat\xx^\star) = \mP(R=1\mid \hat\mX^\star=\hat\xx^\star)  \xlongequal{\text{Thm.~\ref{thm:generator}}} \mP(R=1).
$$
\end{remark}
Thus $D^\star$ recovers the probability $\mP(R=1)$ on every input.

%However, when several missingness patterns occur, the objective depends on $G$ only through the mixture $p_{\hat\mX}(\cdot \mid R = 0)$. Consequently, the optimal generator is not unique. In the nonparametric setting, this set contains the solutions when $\hat X^\star \perp\!\!\!\perp \mM$. Henceforth, $(G^\star, D^\star)$ denotes an optimum with $\hat X^\star \perp\!\!\!\perp \mM$.
\textcolor{black}{
However, when multiple missingness patterns occur, the objective
depends on $G$ only through the mixture distribution
$p_{\hat\mX}(\cdot\mid R=0)$. Consequently, distributional
alignment does not necessarily imply alignment for each
missingness pattern. We therefore introduce the following condition.
}
\textcolor{black}{
\begin{condition}[Pattern Independence]
\label{cond:independence}
The completed random variable under the optimal generator
$G^\star$ is independent of the missingness pattern $\mM$,
i.e.,
\begin{equation}\label{eq:pind}
    \hat\mX^\star \perp\!\!\!\perp \mM.
\end{equation}
\end{condition}}
\textcolor{black}{
Under MCAR, Condition~\ref{cond:independence} together with
Cor.~\ref{cor:match} ensures that
\begin{equation}\label{eq:dind}
p_{\hat\mX^\star}(\uu\mid\mM=\mathbf{m})
= p_{\mX}(\uu), \quad\text{for all } \uu \in \mathbb{R}^d
\end{equation}
for every missingness pattern $\mathbf{m}$ with positive probability.
This yields the correct conditional distribution of the missing
features given the observed features for each pattern.} Furthermore, since Eq.~\eqref{eq:dind} implies the condition in Thm.~\ref{thm:generator}, there exists a global optimizer of $V(G,D)$ that satisfies Condition~\ref{cond:independence}. In particular, when $V(G,D)$ has a unique global optimizer, Condition~\ref{cond:independence} is always satisfied.

\subsubsection{Predictive Behaviour at Optimality}\label{subsec:classifier}

The distributional alignment achieved by the optimal generator has a direct implication on the output of the discriminator. Throughout this section, we consider an incomplete 
observation $\tilde\xx_0$ with the missing mask $\mathbf{m}_0 \neq \unit$, which 
partitions the coordinates into observed indices 
$\obs := \{j : (\mathbf{m}_0)_j = 1\}$ and missing indices 
$\miss := \obs^c$. Note that $\obs$ and $\miss$ are determined by 
$\mathbf{m}_0$ alone and are properties of the datapoint. Let $\hat\mX_\miss^\star = [G^\star(\tilde\xx_0, \mathbf{m}_0, \mZ)]_\miss$ denote the corresponding imputation. 

\paragraph{The optimal predictor with IICN}
Substituting any realisation $\hat\xx^\star_\miss$ of $\hat\mX^\star_\miss$ into Eq.~\eqref{eq:normalized} at $\hat\xx=(\xx_\obs, \hat\xx^\star_\miss)$:\begin{equation}\label{eq:pointwise}\frac{D_i^\star(\xx_\obs, \hat\xx_\miss^\star)}{\sum_jD_j^\star(\xx_\obs, \hat\xx^\star_\miss)} = \mP(Y=i\mid \mX=(\xx_\obs, \hat\xx^\star_\miss)).\end{equation}\textcolor{black}{That is, IICN treats the imputed value $(\xx_\obs, \hat\xx_\miss^\star)$ 
as if it were a complete observation and returns the corresponding Bayes 
posterior. 
Taking expectation over $\hat\mX_\miss^\star$ in
Eq.~\eqref{eq:pointwise}, under
Condition~\ref{cond:independence}, we have}
\textcolor{black}{
\[
\hat\mX_\miss^\star
\mid (\hat\mX_\obs^\star=\xx_\obs,\mM=\mathbf{m}_0)
\stackrel{d}{=}
\mX_\miss\mid\mX_\obs=\xx_\obs.
\]
Thus, marginalizing out the missing coordinates yields the
Bayes-optimal classifier under missingness, as stated below.}
\textcolor{black}{
\begin{theorem}[Bayes-Optimal Classifier]
\label{thm:classifier}
Under MCAR, if Condition~\ref{cond:independence}
is satisfied at the optimum $(G^\star,D^\star)$, then
for any missingness pattern
with positive probability,
\begin{equation}\label{eq:classifier}
\mathbb{E}_{\hat\mX_\miss^\star \mid \xx_\obs}\!\left[
\frac{D_i^\star(\xx_\obs, \hat\mX_\miss^\star)}
     {\sum_{j=1}^n D_j^\star(\xx_\obs, \hat\mX_\miss^\star)}
\right]
= \mP(Y=i \mid \xx_\obs).
\end{equation}
\end{theorem}}
\textcolor{black}{
The proof is given in Appendix~\ref{sec:opt_c}. This shows that suitable optimization of $V(G,D)$ can lead to Bayes-optimal classification under missingness without training a separate classifier: $G^\star$ supplies the correct pattern-wise conditional imputation distribution, $D^\star$ supplies the complete-data class posterior, and averaging over the imputations marginalizes out the missing coordinates.}

\subsection{IICN Algorithm}\label{sec:algo}
\begin{algorithm}[t]
\caption{IICN training procedure with alternating discriminator updates.}
\label{alg:IICN}
\setlength{\abovedisplayskip}{2pt}
\setlength{\belowdisplayskip}{2pt}
\begin{algorithmic}[1]\setlength{\itemsep}{0pt}
\State \textbf{Input:} $\mathcal{D}=\{(\tilde\xx,\mm,y)\}$; phase lengths $L_1, L_2$; learning rates $\eta_{\mathrm{d}}, \eta_{\mathrm{c}}, \eta_{\mathrm{g}}$; regularization weight $\lambda$; noise prior $P_{\mZ}$; number of labels $n$
\State \textbf{Initialize} generator parameters $\boldsymbol{\theta}$ and discriminator parameters $\boldsymbol{\omega}$
\State Set period $H \gets L_1 + L_2$
\For{$t = 1$ to $T$}
\State Sample minibatch $\mathcal{B}=\{(\tilde{\xx}_i,\mm_i,y_i)\}_{i=1}^B \subset \mathcal{D}$
\State Sample noise $\zz_i \sim P_{\mZ}$ and impute $\hat\xx_i \gets \tilde\xx_i \odot \mm_i + G_{\boldsymbol{\theta}}(\tilde\xx_i,\mm_i,\zz_i)\odot(\unit-\mm_i)$
\State Completeness indicator $r_i \gets \mathbb{I}\{\mm_i=\unit\}$
\State  Create pseudo-labels:
        $y_i^{\mathrm{adv}} \gets r_i \cdot y_i + (1-r_i)\cdot (n{+}1)$

\If{$(t \bmod H) < L_1$}\hfill \textnormal{\small\color{blue!60}[Phase A: Adversarial update]}\\
    \colorbox{blue!5}{
    \begin{minipage}{0.97\linewidth}
        \Statex\hspace*{\algorithmicindent} Compute adversarial loss:
        $\mathcal{L}_{\mathrm{adv}}(\boldsymbol{\omega}) \gets \frac{1}{B}\sum_{i=1}^B 
        \mathrm{CE}\!\left(D_{\boldsymbol{\omega}}(\hat\xx_i),\, y_i^{\mathrm{adv}}\right)$
    \Statex\hspace*{\algorithmicindent} Update discriminator:
        $\boldsymbol{\omega} \leftarrow \boldsymbol{\omega} - \eta_{\mathrm{d}} \nabla_{\boldsymbol{\omega}}\mathcal{L}_{\mathrm{adv}}(\boldsymbol{\omega})$
    \end{minipage}
    }
  \Else\hfill \textnormal{\small\color{blue!60}[Phase B: Classification update]}\\
    \colorbox{blue!15}{
    \begin{minipage}{0.97\linewidth}
    \Statex\hspace*{\algorithmicindent} Compute classification loss:
        $\mathcal{L}_{\mathrm{cls}}(\boldsymbol{\omega}) \gets \frac{1}{B}\sum_{i=1}^B 
        \mathrm{CE}\!\left(D_{\boldsymbol{\omega}}(\hat\xx_i),\, y_i\right)$
    \Statex\hspace*{\algorithmicindent} Update discriminator:
        $\boldsymbol{\omega} \leftarrow \boldsymbol{\omega} - \eta_{\mathrm{c}} \nabla_{\boldsymbol{\omega}}\mathcal{L}_{\mathrm{cls}}(\boldsymbol{\omega})$
    \end{minipage}
    }
 \EndIf
\State Compute generator loss on incomplete samples:
$$\mathcal{L}_{\mathrm{gen}}(\boldsymbol{\theta})\gets\frac{1}{B}\sum_{i=1}^{B} (1-r_i)\;
\Big(
       - \mathrm{CE}\!\left(D_{\boldsymbol{\omega}}(\hat\xx_i),\, y_i^{\mathrm{adv}}\right)
        + \lambda\cdot\mathrm{MSE}(\hat\xx_i)
\Big)$$
\State Update generator:
$\boldsymbol{\theta} \leftarrow \boldsymbol{\theta} - \eta_{\mathrm{g}} \nabla_{\boldsymbol{\theta}}\mathcal{L}_{\mathrm{gen}}(\boldsymbol{\theta})$
\EndFor
\end{algorithmic}
\end{algorithm}
We have shown above that solving the optimization problem~\eqref{eq:loss} under Condition~\ref{cond:independence} attains Bayes-optimal classification. In implementation, Eq.~\eqref{eq:pind} may be cast as a constraint or as part of the objective in the Lagrange form. Our numerical results suggest that even when we do not explicitly enforce~\eqref{eq:pind} in the IICN algorithm, the learned imputation distributions still exhibit approximate independence from the missingness patterns in practice. We empirically assess the dependence between the completed data and the missingness patterns in Appendix~\ref{app:mask_independence}.

More importantly, while the proposed objective enjoys desirable theoretical properties, directly optimizing it with standard adversarial training often leads to unstable dynamics and degenerate solutions~\citep{goodfellow2014generative, arjovsky2017towards, arjovsky2017wasserstein, heusel2017gans, mescheder2018training}. In our setting, a naively trained discriminator can exploit shortcut cues that separate real incomplete observations from generator imputations, rejecting imputed samples without assessing semantic plausibility~\citep{ilyas2019adversarial, clarysse2022adversarial, mctavish2024interpretable}. This failure mode undermines both semantic classification and the quality of the feedback provided to the generator.

To address this challenge, we introduce a two-phase alternating training strategy that enforces task-aligned adversarial behavior by integrating classification signals into discriminator optimization, thus preventing the discriminator from collapsing to a trivial imputation detector. The overall training procedure is summarized in Algorithm~\ref{alg:IICN}. In this algorithm, we consider the parametrized discriminator and generator, $D_{\boldsymbol{\omega}}$ and $G_{\boldsymbol{\theta}}$. For a minibatch of size $B$, let $i \in \{1,\ldots,B\}$ index samples in the minibatch, and let $r_i \gets \mathbb{I}\{\mm_i=\mathbf{1}\}$ indicate whether sample $i$ is complete. Each sample is completed by $\hat\xx_i \gets \tilde\xx_i \odot \mm_i + G_{\boldsymbol{\theta}}(\tilde\xx_i,\mm_i,\zz_i)\odot(1-\mm_i)$, where $\zz_i\sim P_{\mZ}$.

\paragraph{Two-phase alternating optimization of the discriminator.}Training of the discriminator proceeds in periods of length $H=L_1+L_2$. At iteration $t$, we perform an adversarial discriminator update if $(t \bmod H) < L_1$, and a supervised classification update otherwise. We construct an adversarial label that treats incomplete samples as an additional class: $y_i^{\mathrm{adv}} = r_i y_i + (1-r_i)(n{+}1).$
%$\hat\xx_i \gets \tilde\xx_i \odot \mm_i + G(\tilde\xx_i,\mm_i,\zz_i)\odot(1-\mm_i)$

\emph{Phase A (adversarial update).}
 The discriminator is 
updated via gradient descent on the cross-entropy loss (Lines 10):
 \setlength{\abovedisplayskip}{5pt}
\setlength{\belowdisplayskip}{5pt}
\setlength{\abovedisplayshortskip}{5pt}
\setlength{\belowdisplayshortskip}{5pt}
\[
\mathcal{L}_{\mathrm{adv}}(\boldsymbol{\omega}) \gets \frac{1}{B}\sum_{i=1}^B 
        \mathrm{CE}\!\left(D_{\boldsymbol{\omega}}(\hat\xx_i),\, y_i^{\mathrm{adv}}\right).
\]
\emph{Phase B (classification update).}
To prevent the discriminator from collapsing into a pure imputation 
detector, we periodically substitute the adversarial label with the 
true semantic label, applying gradient descent on the classification 
loss (Lines 12):
 \setlength{\abovedisplayskip}{5pt}
\setlength{\belowdisplayskip}{5pt}
\setlength{\abovedisplayshortskip}{5pt}
\setlength{\belowdisplayshortskip}{5pt}
\[
\mathcal{L}_{\mathrm{cls}}(\boldsymbol{\omega}) \gets \frac{1}{B}\sum_{i=1}^B 
        \mathrm{CE}\!\left(D_{\boldsymbol{\omega}}(\hat\xx_i),\, y_i\right).
\]
This alternating schedule discourages persistent reliance on artifact-based shortcut features and mitigates early discriminator domination.
\vspace{-3mm}
\paragraph{Generator optimization with reconstruction regularization.}
The generator is trained only on incomplete samples (i.e., $r_i=0$) to produce imputations that both fool the discriminator away from the ``incomplete'' class and remain consistent with observed entries. That is, the generator is updated via gradient descent on the following loss (Line 13):
 {
\setlength{\abovedisplayskip}{5pt}
\setlength{\belowdisplayskip}{5pt}
\setlength{\abovedisplayshortskip}{5pt}
\setlength{\belowdisplayshortskip}{5pt}
\[
\mathcal{L}_{\mathrm{gen}}(\boldsymbol{\theta})\gets\frac{1}{B}\sum_{i=1}^{B} (1-r_i)\;
\Big(
       - \mathrm{CE}\!\left(D_{\boldsymbol{\omega}}(\hat\xx_i),\, y_i^{\mathrm{adv}}\right)
        + \lambda\cdot\mathrm{MSE}(\hat\xx_i)
\Big),
\]}
where the reconstruction term enforces agreement on observed coordinates:
 {
\setlength{\abovedisplayskip}{5pt}
\setlength{\belowdisplayskip}{5pt}
\setlength{\abovedisplayshortskip}{5pt}
\setlength{\belowdisplayshortskip}{5pt}
\begin{equation}\label{eq:mse}
\mathrm{MSE}(\hat{\xx})
\;=\;
\frac{\left\|\big(\tilde{\xx}-G_{\boldsymbol{\theta}}(\tilde{\xx},\mm,\zz)\big)\odot \mm\right\|_2^2}{\|\mm\|_2^2}.
\end{equation}}

\section{Experiments}
We compare IICN against a diverse set of baselines covering 
distinct strategies for classification under missingness. Zero imputation~(Zero) and mean imputation~(Mean) are simple baselines that do not explicitly model missingness, with the former reflecting performance without any imputation reasoning and the latter using a global statistic. GAIN~\citep{gain} represents adversarial imputation methods applied as a downstream-task-independent preprocessing step. SupMIWAE~\citep{SupMIWAE} and LAC~\citep{sugiyama2025discrete, sugiyama2025unified} are recent joint optimization methods that align imputation with the classification objective: SupMIWAE via importance-weighted variational inference, and LAC via alternating minimization of a weighted loss that combines the prediction error and reconstruction error. As a reference, we also include complete-only ($\text{CO}^\dagger$) training, which discards incomplete datapoints and uses only fully observed data in both training and testing. Together, these baselines span the major categories of missing-data classification strategies, allowing us to evaluate IICN against both classical preprocessing approaches and state-of-the-art end-to-end methods. All code and experimental results are publicly available on GitHub.\footnote{\url{https://github.com/yuel5304/IICN}}
\subsection{Experimental Setup}
We evaluate IICN on three dataset categories progressing from controlled to realistic missingness: (i) synthetic Burger dataset (following \citet{SupMIWAE}) for validation under known 
data-generating processes; (ii) FashionMNIST~\citep{xiao2017fashion} and CIFAR-10~\citep{krizhevsky2009learning} with artificial missingness, allowing controlled missing rates on high-dimensional real data; and (iii) UCI~\citep{asuncion2007uci} tabular datasets with naturally occurring missing values for realistic evaluation, including Mammographic Mass, Credit Approval, Diabetes, Hepatitis, and Thyroid.\footnote{For the tabular datasets, we drop ID, free-text, and high-missingness ($>80\%$) columns, and integer-encode all categorical features, factorize labels.} 

\paragraph{Training details} The IICN discriminator and the classifiers of competing methods share the same backbone, with task-specific input and output layers. We use a two-layer fully connected network with ReLU activation and 50 hidden units for the synthetic Burger dataset, ResNet-8 for CIFAR-10, LeNet for FashionMNIST, and a four-layer fully connected network with ReLU activation and 256 units per hidden layer for the UCI tabular datasets. The IICN discriminator has $n+1$ outputs, corresponding to the $n$ original classes and an additional ``imputed'' class. The generator uses an MLP backbone for synthetic and tabular data and a convolutional backbone for image data, with output layers tailored to imputation. Detailed architectures are provided in Appendix~\ref{app:model}. For Monte Carlo marginalization during testing (see Appendix~\ref{app:inference} for details), we use $K=50$ samples on synthetic data and $K=1$ on image datasets for both SupMIWAE and IICN. Note that due to the relatively small sample sizes of tabular datasets, 
training deep models in this setting is prone to high variance. To stabilize evaluation, we adopt a $k$-fold cross-validation ensemble, where for each method, we train $k$ separate models on different training/validation folds and average their predictions at test time. These strategies help stabilize training and evaluation in small-sample regimes with substantial missingness~\citep{hastie2009elements,guo2017calibration, maddox2019simple}. More training details are included in Appendix~\ref{app:exp}.

\subsection{Numerical Results}
\paragraph{Synthetic dataset}
For the synthetic Burger dataset, we set the second dimension missing at random. In this case, $p(\xx_2|\xx_1)$ is multimodal under missingness. This setting challenges single-point imputation methods~(zero, mean, LAC) and provides a controlled test of predictive marginalization. The Bayes-optimal classifier for this case always satisfies $\mP(Y=1 \mid \xx_1)=2/3$, reflecting the intrinsic ambiguity induced by incomplete observations. 

Figure~\ref{fig:burger} illustrates qualitative differences across methods. The top row shows a representative example with missing rate equal to $0.9$. Zero and mean imputation collapse missing values to a single fixed point (zero or the global average), forcing all imputed samples into a narrow region and producing distorted decision boundaries. GAIN partially recovers the data distribution but exhibits pronounced ambiguity near the class boundary, resulting in unstable predictions. SupMIWAE produces smoother decision regions  through importance-weighted sampling, yet remains sensitive to the intrinsic uncertainty induced by missing values. LAC's deterministic estimation fails in such multimodal case, concentrating imputed samples in a single mode of the conditional distribution. In contrast, IICN spans both modes and yields decision boundaries closest to the Bayes-optimal classifier. 

The bottom row quantifies these behaviors across different missing rates. Single-imputation methods such as zero imputation, mean imputation, GAIN and LAC exhibit substantial deterioration in performance on complete datapoints as the missing rate increases, while their accuracy on incomplete datapoints improves. But this rise is deceptive: $67\%$ coincides with the marginal class probability $\mP(Y=1)=2/3$, indicating that the classifier has degenerated into a constant majority-class predictor rather than learning to discriminate from imputed inputs. The accompanying drop in accuracy on complete datapoints at missing rate equal to $0.9$ 
confirms this collapse, regardless of input. In contrast, SupMIWAE and IICN remain stable across all missing rates and approach $2/3$ on incomplete datapoints through genuine marginalization rather than degenerate prediction. IICN matches the bound most consistently, empirically validating Eq.~\eqref{eq:classifier}.
\vspace{-2mm}
\begin{figure}[t]
\begin{minipage}{\textwidth}
\centering
\vspace{-0.06\linewidth}
\hspace{-0.02\linewidth}
\begin{subfigure}{0.18\textwidth}
    \centering
    \includegraphics[height=2.6cm]{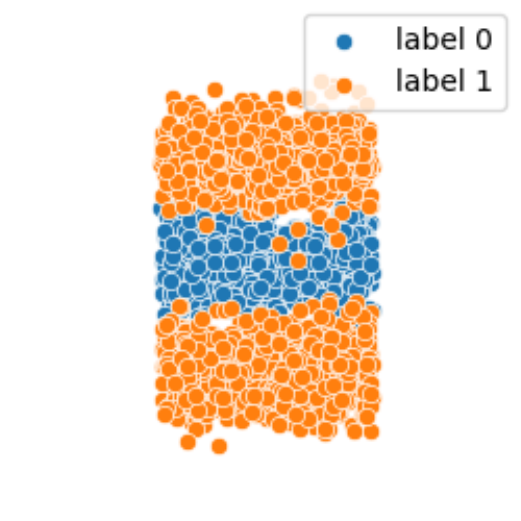}
    \vspace{-0.2\linewidth}
    \caption*{Burger dataset}
\end{subfigure}\hspace{-0.001\linewidth}
\begin{subfigure}{0.13\textwidth}
    \centering
    \includegraphics[width=\textwidth]{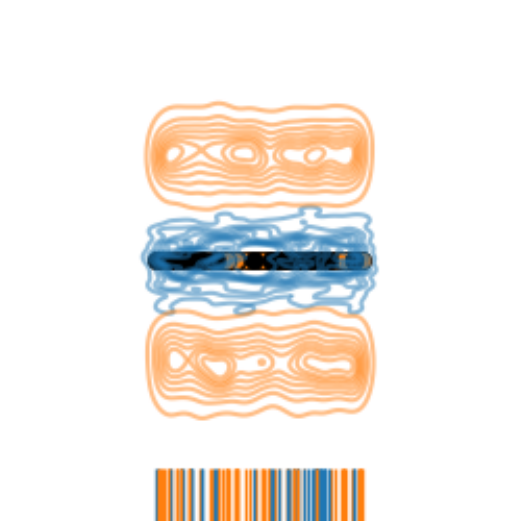}
    \includegraphics[width=\textwidth]{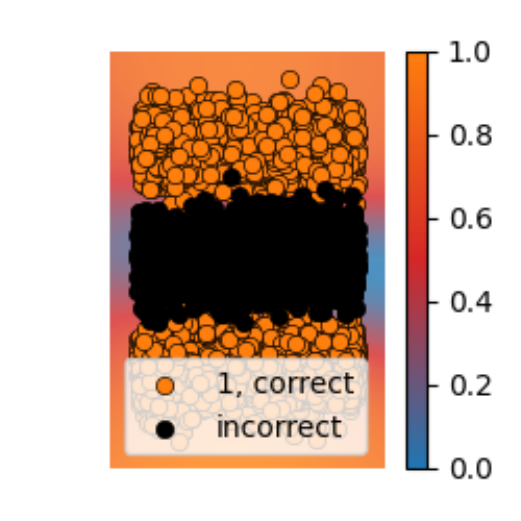}
    \caption{Zero}
\end{subfigure}\hspace{-0.001\linewidth}
\begin{subfigure}{0.13\textwidth}
    \centering
    \includegraphics[width=\textwidth]{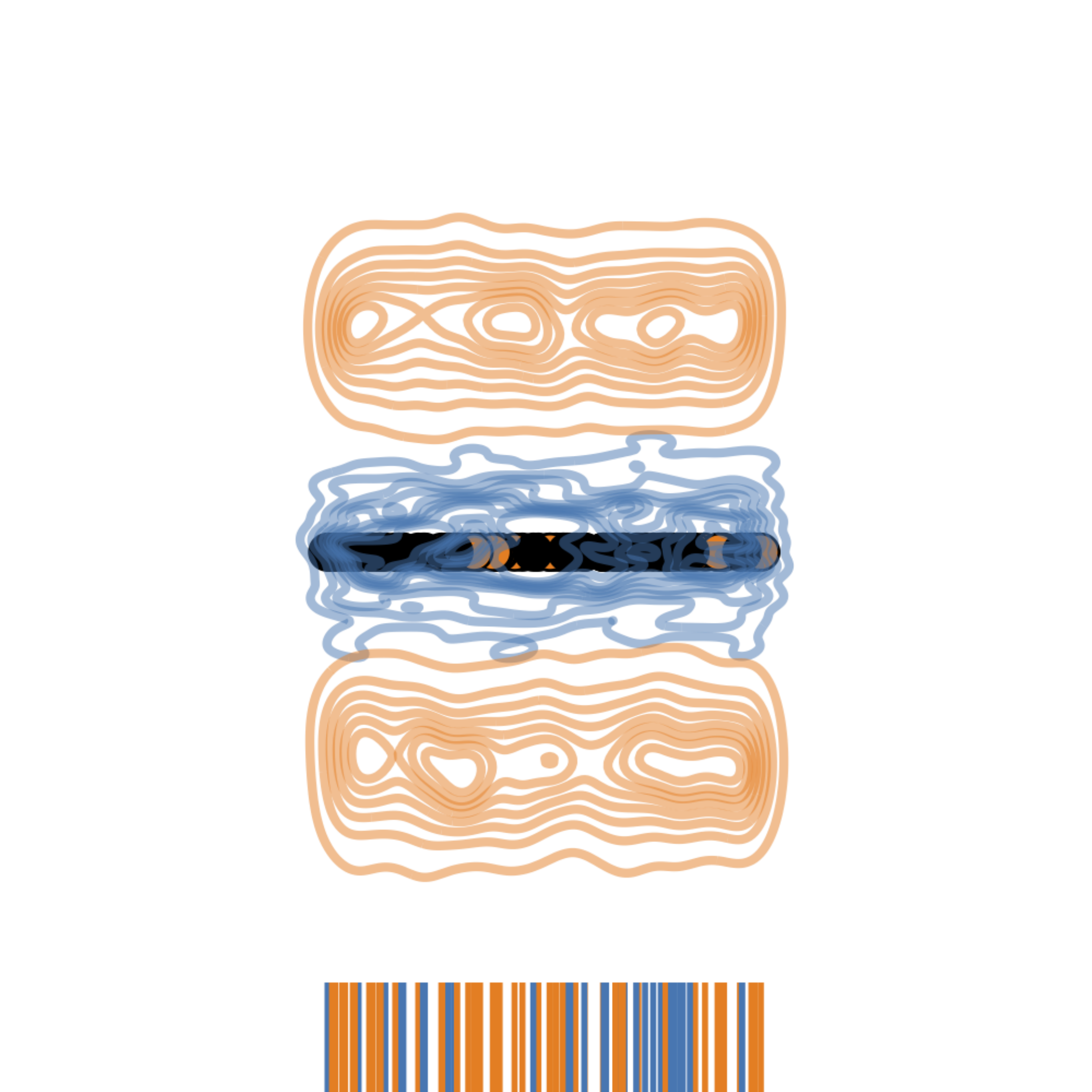}
    \includegraphics[width=\textwidth]{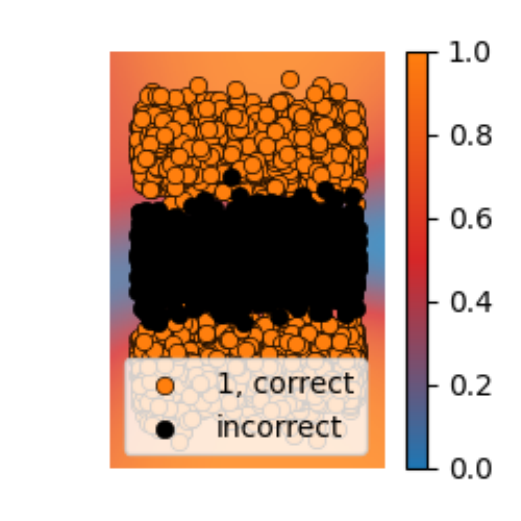}
    \caption{Mean}
\end{subfigure}\hspace{-0.001\linewidth}
\begin{subfigure}{0.13\textwidth}
    \centering
    \includegraphics[width=\textwidth]{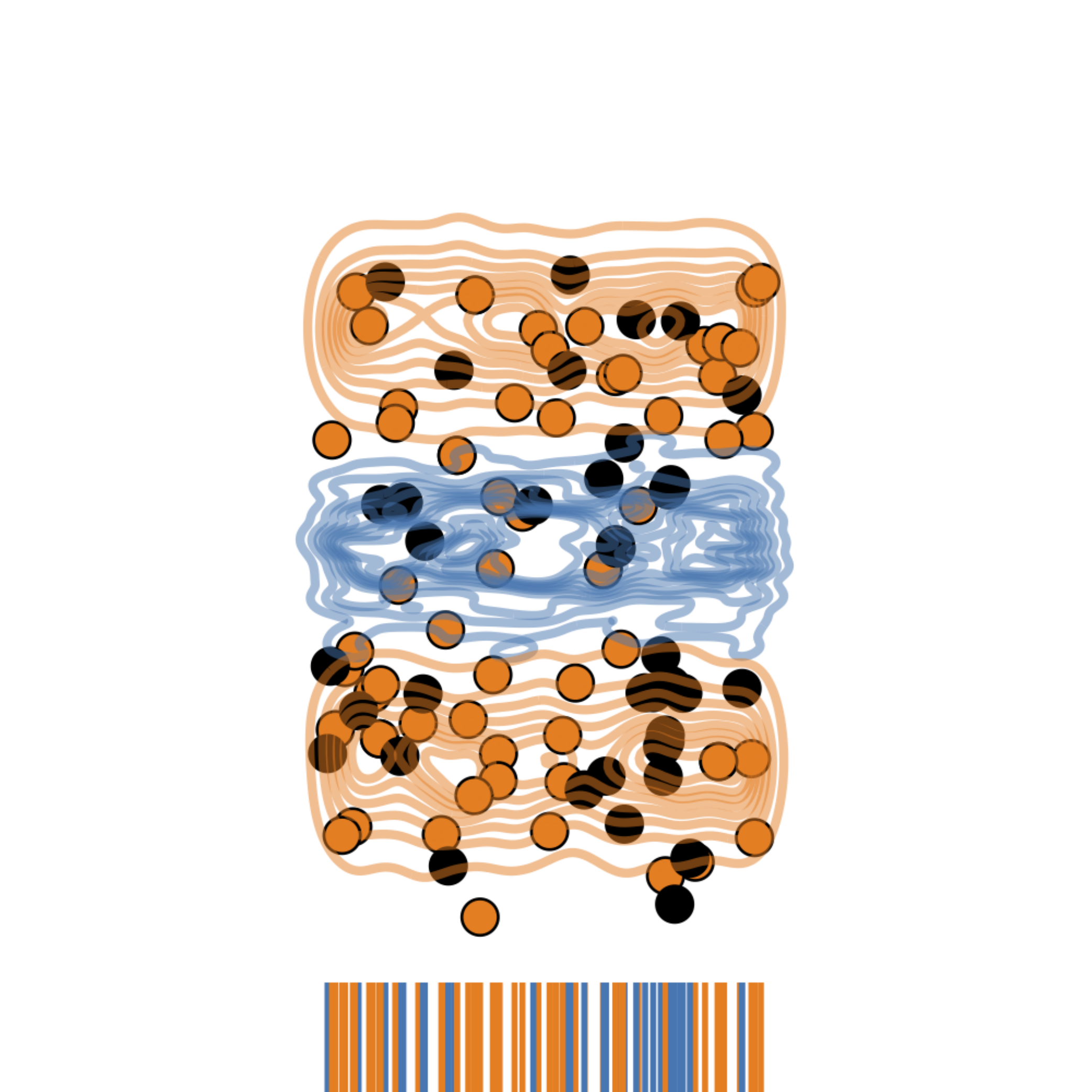}
    \includegraphics[width=\textwidth]{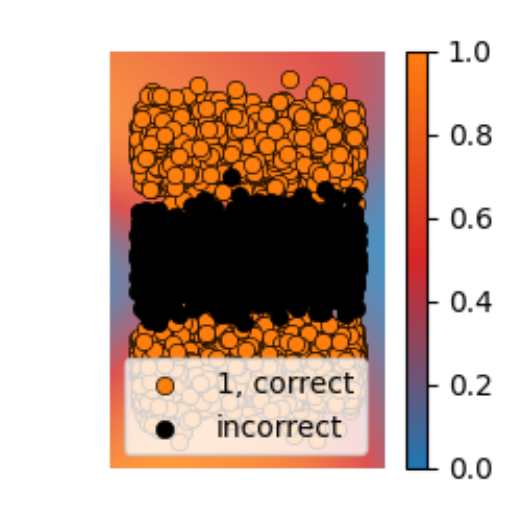}
    \caption{GAIN}
\end{subfigure}\hspace{-0.001\linewidth}
\begin{subfigure}{0.13\textwidth}
    \centering
    \includegraphics[width=\textwidth]{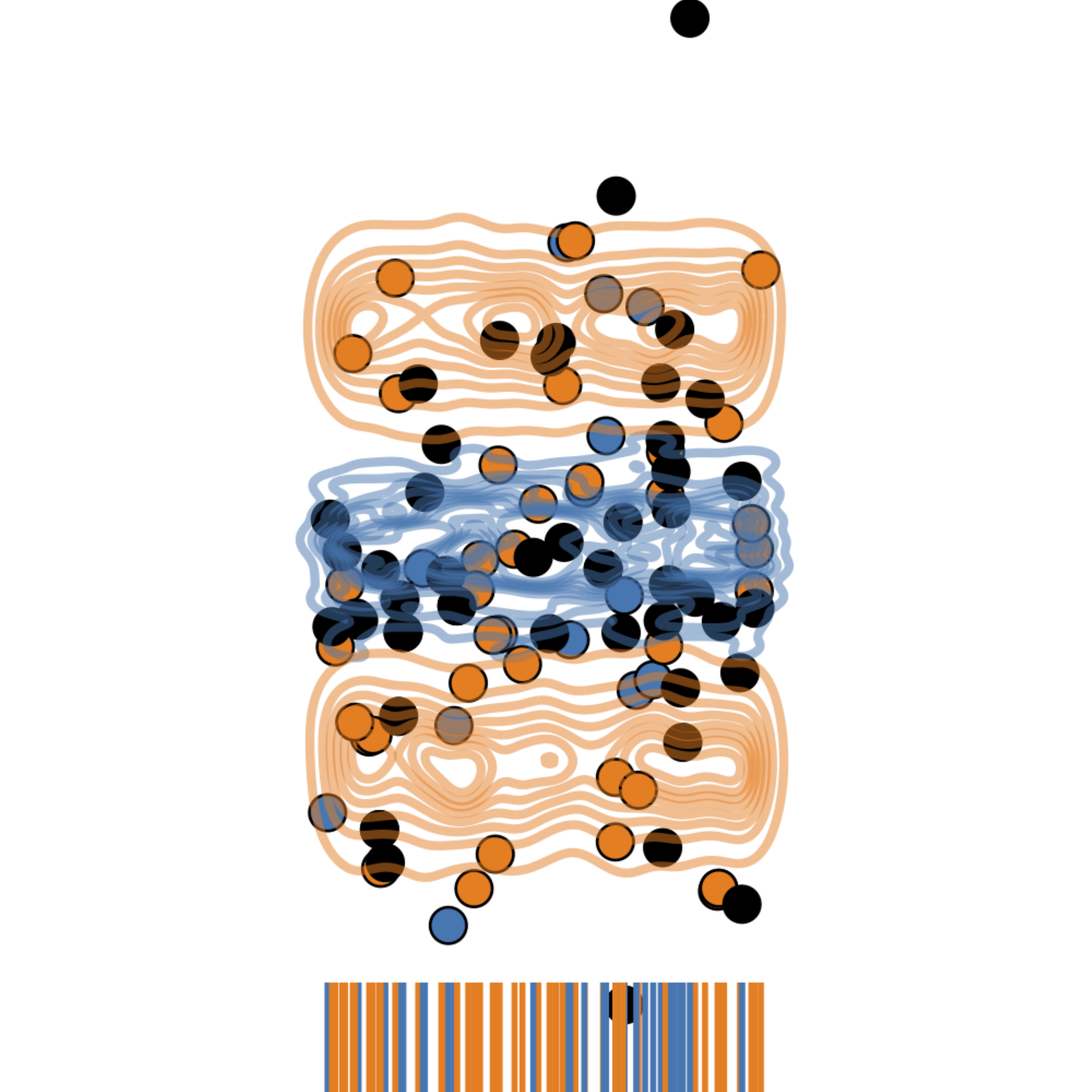}
    \includegraphics[width=\textwidth]{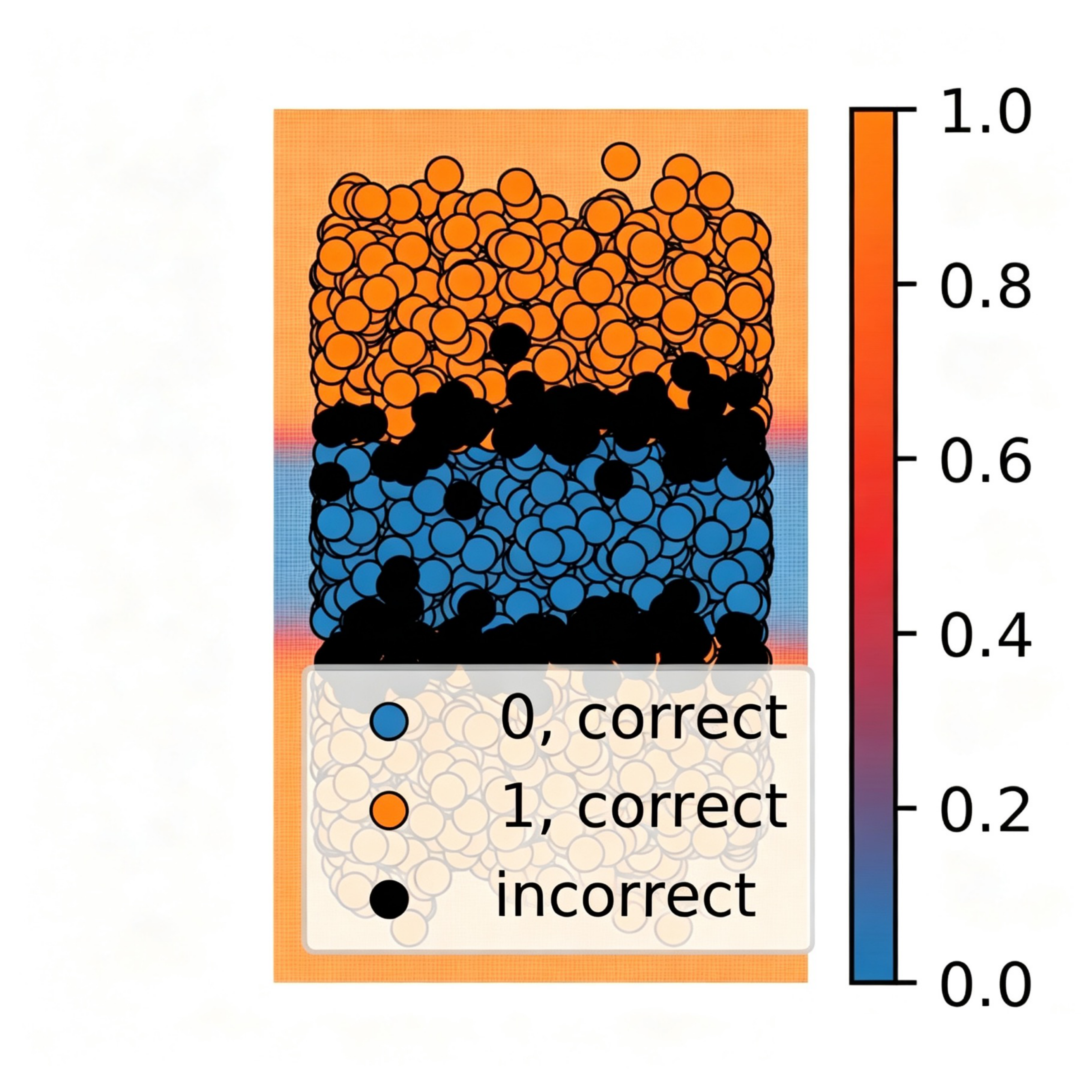}
    \caption{SupMIWAE}
\end{subfigure}\hspace{-0.001\linewidth}
\begin{subfigure}{0.13\textwidth}
    \centering
    \includegraphics[width=\textwidth]{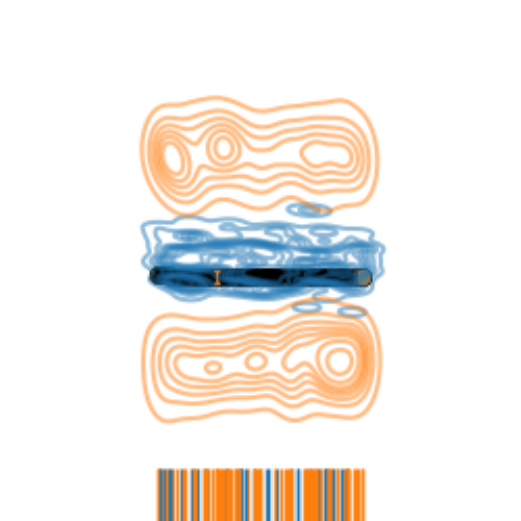}
    \includegraphics[width=\textwidth]{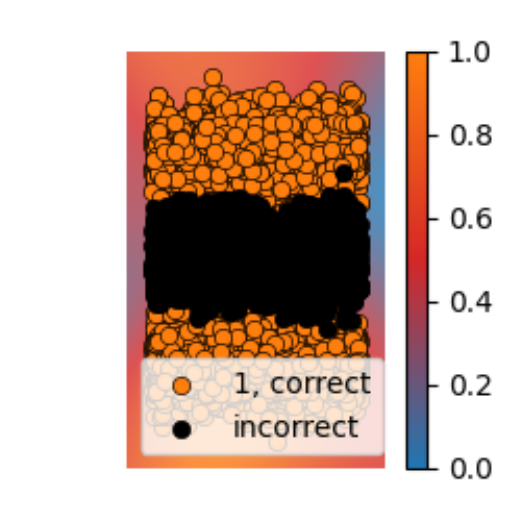}
    \caption{LAC}
\end{subfigure}\hspace{-0.001\linewidth}
\begin{subfigure}{0.13\textwidth}
    \centering
    \includegraphics[width=\textwidth]{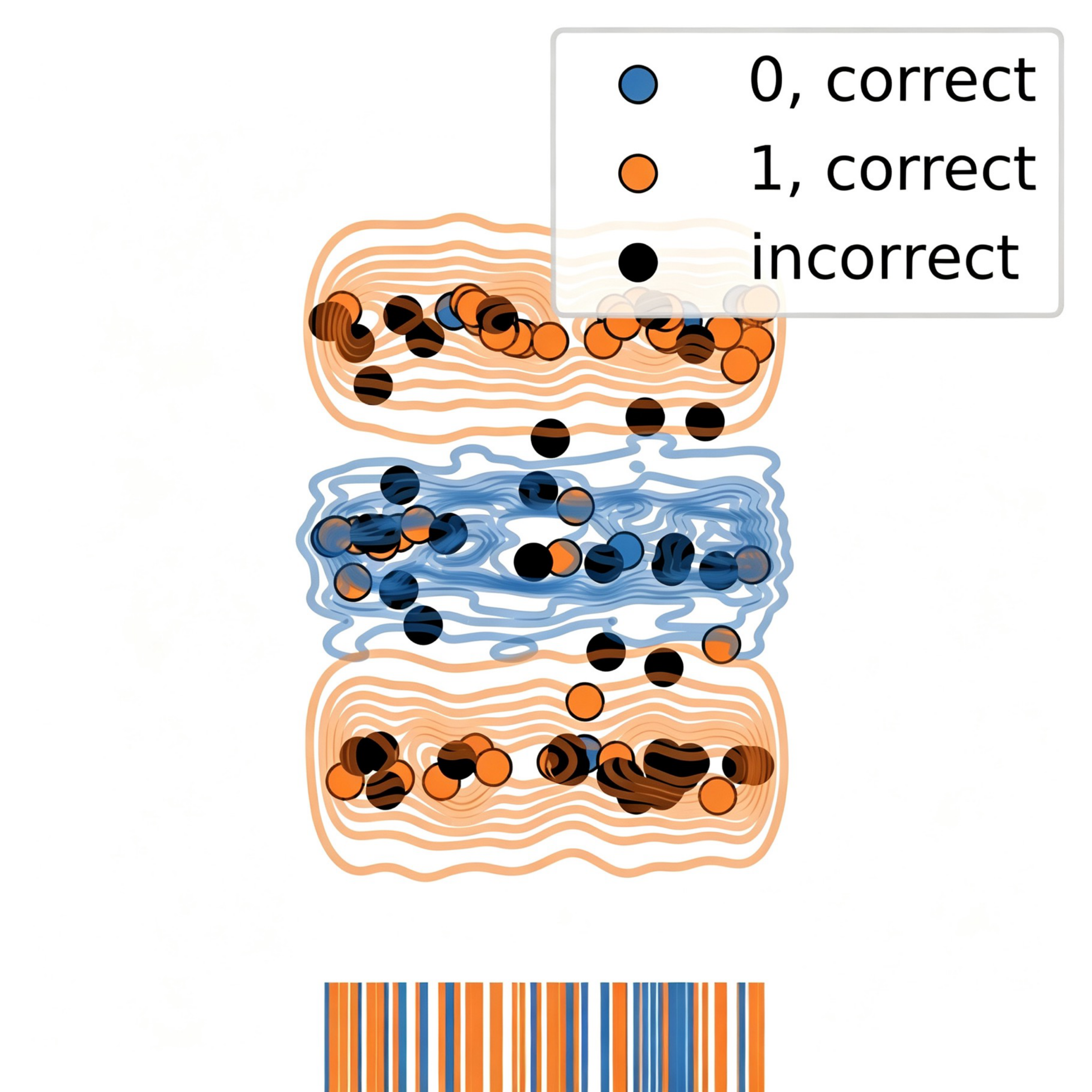}
    \includegraphics[width=\textwidth]{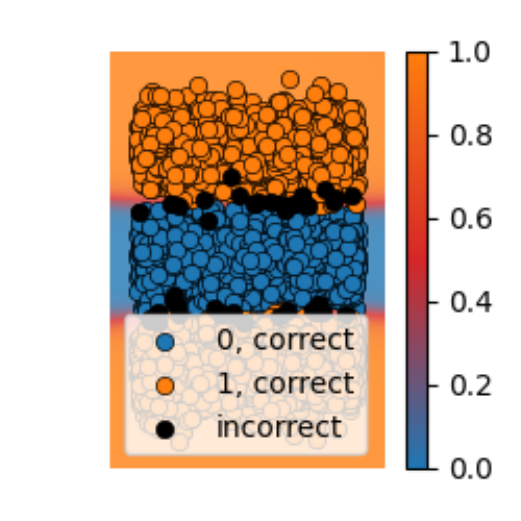}
    \caption{IICN}
\end{subfigure}
\end{minipage}

%\vspace{6mm}

\begin{minipage}{\textwidth}
\centering

% -------- missing rate plot --------
\hspace*{-0.15\linewidth}
\begin{subfigure}[t]{0.47\linewidth}
  %\centering\vspace{0pt}
  \includegraphics[height=3.0cm]{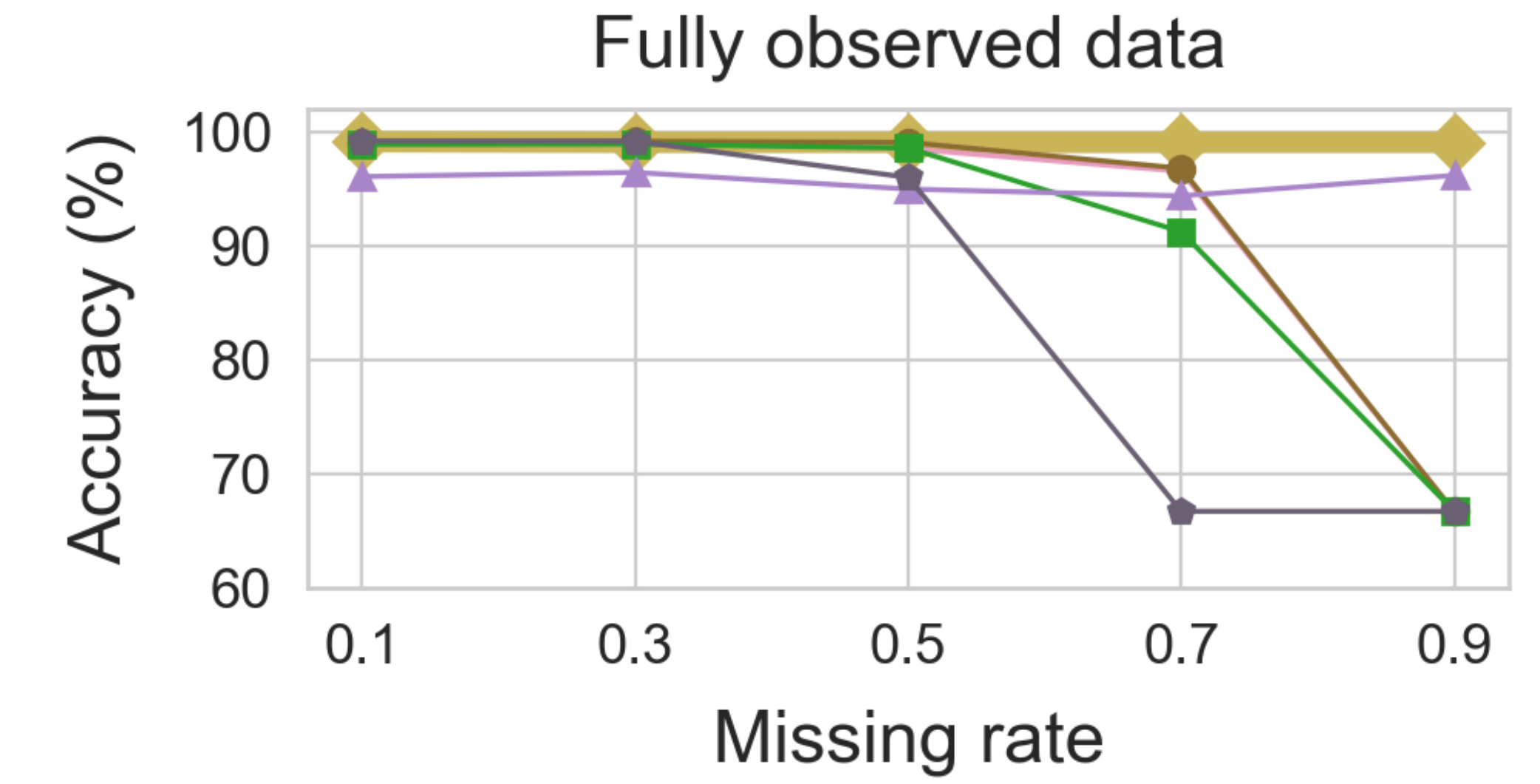}
  %\caption*{\quad\quad\quad\quad Complete data}
\end{subfigure}\hspace{-0.05\linewidth}
\begin{subfigure}[t]{0.47\linewidth}
  %\centering\vspace{0pt}
  \includegraphics[height=3.0cm]{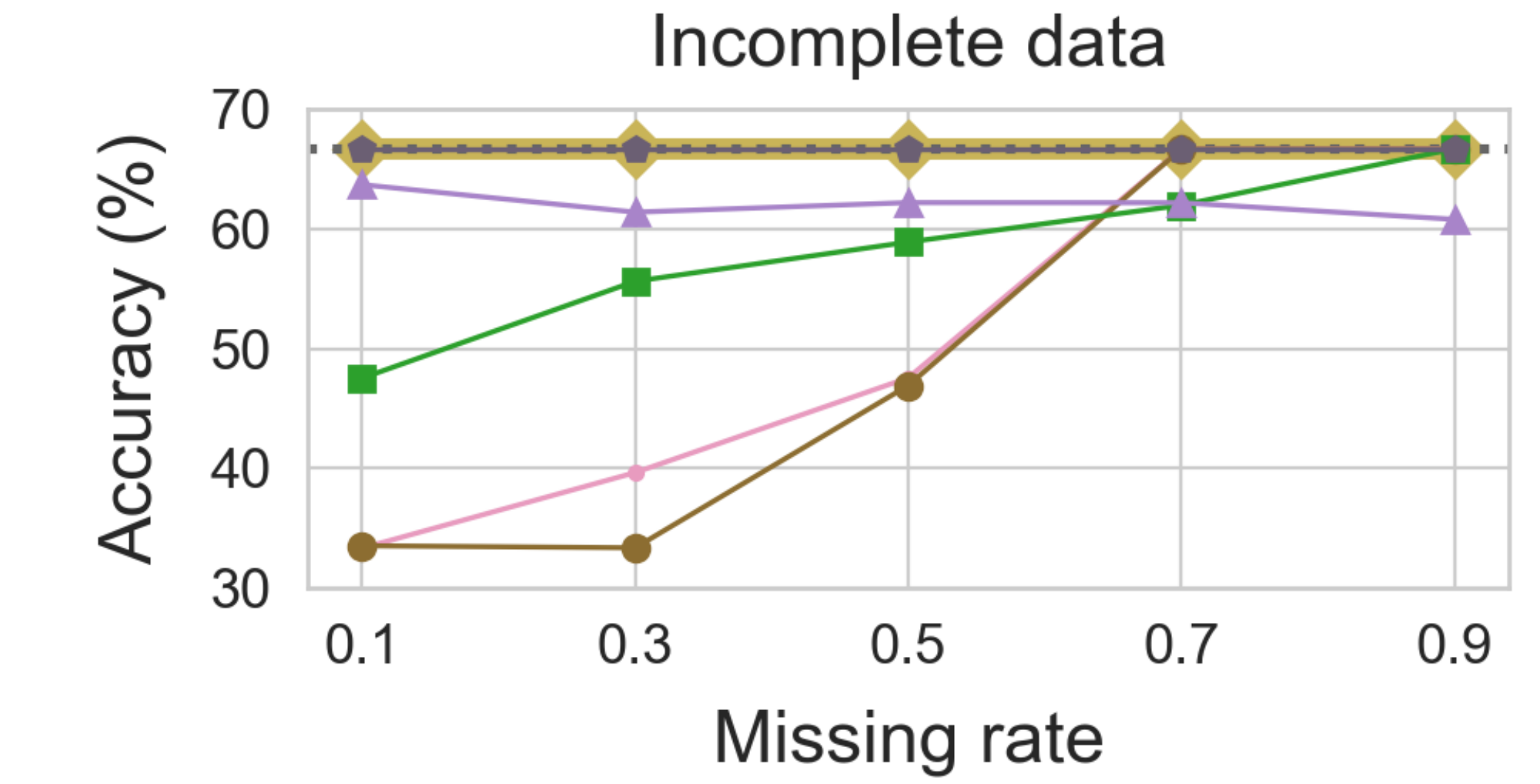}
  %\caption*{\quad\quad\quad\quad Incomplete data}
\end{subfigure}\hspace{-0.03\linewidth}
\begin{subfigure}[t]{0.05\linewidth}
  %\centering\vspace{0pt}
  \includegraphics[height=3.3cm]{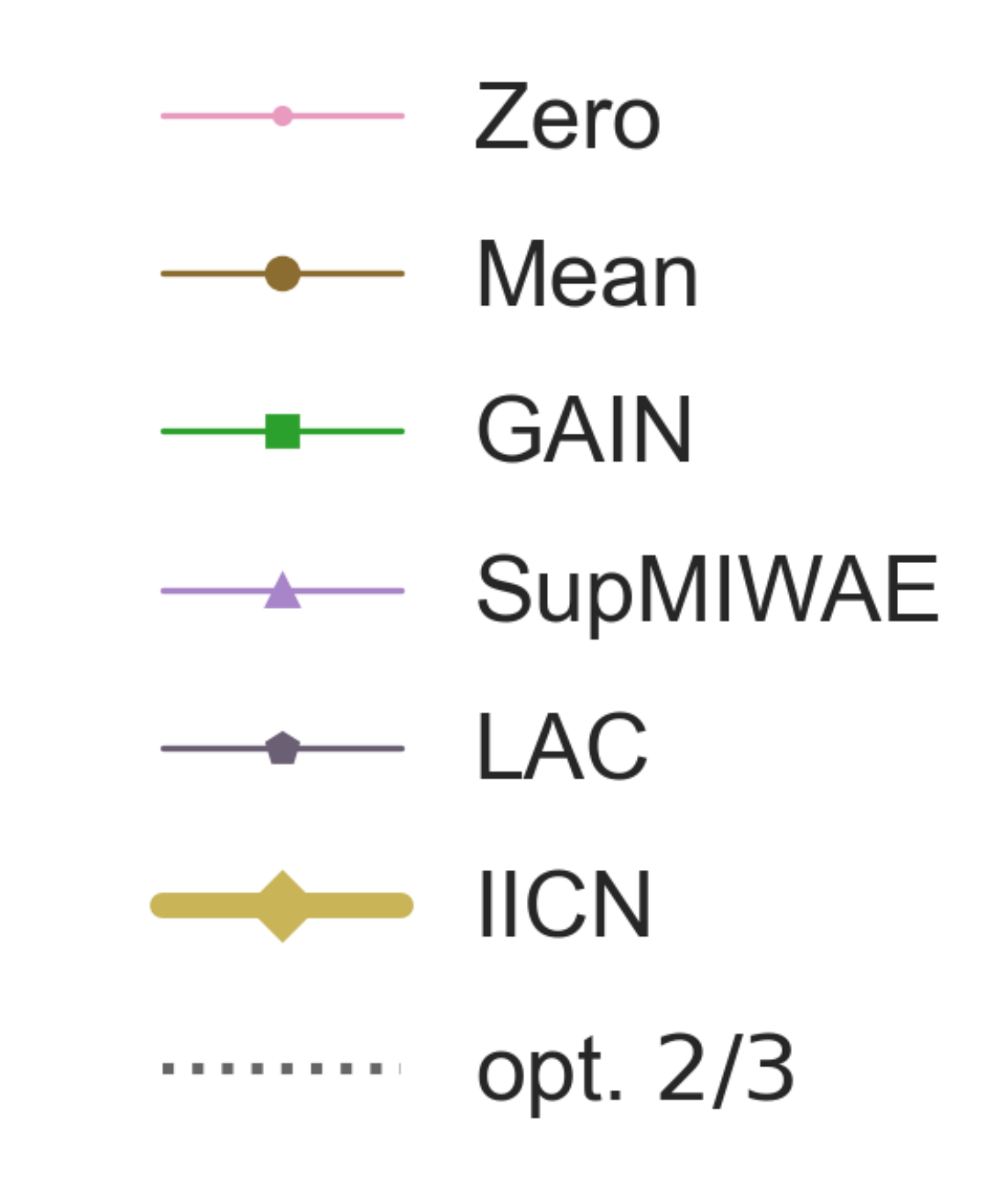}
\end{subfigure}
\end{minipage}
%caption{\textbf{Top row:} Data distribution, decision boundaries, and sample-level visualizations for different methods. The dataset is a two-class problem with $y=0$ (blue) and $y=1$ (orange). Missingness is introduced completely at random along the second dimension with $P(\text{miss})=0.9$. Observed values of incomplete samples are shown as colored sticks at the bottom, and black points indicate misclassified samples. Due to the intrinsic ambiguity of incomplete observations, the Bayes-optimal accuracy on incomplete data is upper-bounded by $P(y=1\mid \xx_{\mathrm{obs}})=2/3$. \\ \textbf{Bottom row:} Classification accuracy on fully observed (left) and incomplete (right) datapoints under varying missing rates $P(\text{miss})\in\{0.1,0.5,0.9\}$.The dashed line indicates the theoretical optimum on incomplete data. Both IICN and SupMIWAE use multiple imputations with $K=50$.}
\caption{Results on the synthetic Burger dataset.
\textbf{Top:} Data distribution, decision boundaries, and sample-level predictions for different imputation methods for an example with missing rate 0.9. The dataset is a two-class problem with $Y = 0$ (blue) and $Y = 1$ (orange), while black points indicate misclassified samples. Observed values of incomplete samples are shown as colored sticks at the bottom. \textbf{Bottom:} Classification accuracy on fully observed (left) and incomplete (right) datapoints under varying missing rates $\in \{0.1, 0.3, 0.5, 0.7, 0.9\}$. The dotted line indicates the Bayes-optimal accuracy on incomplete data.}
\label{fig:burger}
%\vspace{-6mm}
  \end{figure}

%In particular, from the analysis of the Burger dataset, when there is a missing 2nd dimension, the optimal classifier is known as $P(y=1\mid x_1)=\frac{2}{3}$. For a 2-dimensional dataset, having the observed values on only the 1st dimension is not enough to infer the values for the 2nd dimension. There is a great deal of randomness in inferring the values of the 2nd dimension given the values of the 1st dimension.  Thus, those methods claimed to impute the best, e.g., GAIN, instead of recovering the complete datapoint, introduce a great deal of randomness in single imputation and then break the classification performance. In the bottom row of Figure~\ref{fig:burger}, when $P=0.9$, it tends to classify all datapoints to be label 1 and when $P=0.1$, it classifies those incomplete datapoints not much better than random guessing. Classification performance with mean imputation also suffers from the degree of missingness, especially as the classification performance degrades with increasing degree of missingness. 
%However, methods incorporating multiple imputations show consistency regardless of the different degrees of missingness as indicated in the bottom row of Figure~\ref{fig:burger}. Meanwhile, with the same number of multiple imputations, IICN has better classification performance than SupMIWAE. Also note that IICN could ideally yield the test accuracy 0.667 on the Burger dataset with the 2nd dimension missing, which confirms the correctness of our previous theoretical analysis in Equation~\ref{eq: classifier}

\paragraph{Real-world dataset with synthetic missingness}
%We consider introducing different manipulations to manually control the degree of missingness for the real-world dataset. We denote by \( P_{\mathrm{obs}}\) the proportion of complete data points in the dataset. For incomplete data points, each feature is independently missing with probability \( P_{\mathrm{miss}}\). 
For FashionMNIST and CIFAR-10, we control the proportion of complete datapoints via $P_{\mathrm{obs}}$, and independently mask each dimension of incomplete datapoints with probability $P_{\mathrm{miss}}$. Tables~\ref{tab:fashionmnist} and~\ref{tab:cifar10} report classification accuracy on FashionMNIST and CIFAR-10 \textcolor{black}{under multiple missing rates from mild to severe.} Across both datasets, IICN mostly achieves the highest accuracy on incomplete datapoints, substantially outperforming observed-only baselines and classification-independent imputation methods. The performance gap becomes more pronounced as the proportion of missing data increases, highlighting the benefit and robustness of task-aligned adversarial training under severe missingness, \textcolor{black}{even though we also observe that IICN's performance degrades gracefully as the missing rate increases. }%On fully observed datapoints, all methods exhibit comparable performance, indicating that IICN does not degrade classification accuracy when complete information is available. The consistent trends observed across both FashionMNIST and CIFAR-10 demonstrate that the advantages of IICN generalize to high-dimensional image data and are robust across different missingness configurations. Meanwhile, despite consistently achieving the best classification performance, IICN does not always yield the lowest MSE in image reconstruction (Tables~\ref{tab:fashion_mse},~\ref{tab:cifar_mse} in Appendix~\ref{app:experiment}), indicating that imputation quality measured by MSE is not necessarily aligned with predictive accuracy.

\begin{table}
\setlength\tabcolsep{4pt}
\renewcommand{\arraystretch}{1.1}
\resizebox{\textwidth}{!}{%
\begin{tabular}{cc|>{\columncolor{gray!10}}c
>{\columncolor{gray!10}}c
>{\columncolor{gray!10}}c
>{\columncolor{gray!10}}c
>{\columncolor{gray!10}}c
>{\columncolor{gray!10}}c
|cccccc
c}
\toprule
\multicolumn{2}{c|}{\textbf{Missing Rate}} &
\multicolumn{6}{c|}{\cellcolor{gray!10}{\textbf{Incomplete data}}}&
\multicolumn{7}{c}{\textbf{Complete data}}
\\
\textbf{$P_\obs$} & \textbf{$P_\miss$}
& Zero & Mean & GAIN & SupMIWAE & LAC & IICN
& Zero & Mean & GAIN & SupMIWAE & LAC &  IICN & \textit{CO}$^\dagger$ \\
\midrule
\multirow{2}{*}{$0.1$}
& $0.5$ &$83.10\scriptstyle\pm.70$&$85.69\scriptstyle\pm.70$ & $86.40\scriptstyle\pm.90$ & $86.81\scriptstyle\pm.70$ &
$85.88\scriptstyle\pm.70$
& \cellcolor{gray!35}$\textbf{88.01}\scriptstyle\pm.60$  &
$86.97\scriptstyle\pm.70$
&$87.27\scriptstyle\pm.68$&$\textbf{87.38}\scriptstyle\pm.80$&\cellcolor{gray!35}$\textbf{89.20}\scriptstyle\pm1.90$
&$87.01 \scriptstyle\pm.70$
&$\textbf{89.13}\scriptstyle\pm.60$
& \multirow{2}{*}{$86.38\scriptstyle\pm.70$} \\
&$0.8$&$78.45\scriptstyle\pm.80$&$79.78\scriptstyle\pm.80$ & $81.18\scriptstyle\pm.80$ & \cellcolor{gray!35}$\textbf{83.40}\scriptstyle\pm.80$
&$81.64\scriptstyle\pm.80$
& $\textbf{83.06}\scriptstyle\pm.70$  &
$84.64\scriptstyle\pm.70$&$84.24\scriptstyle\pm.70$&$86.09\scriptstyle\pm.70$&\cellcolor{gray!35}$\textbf{88.30}\scriptstyle\pm2.00$
&$84.46\scriptstyle\pm.70$
&$86.06\scriptstyle\pm.70$ &  \\
\midrule
\multirow{2}{*}{$0.5$}
& $0.5$ &  $83.00\scriptstyle\pm.70$&$83.89\scriptstyle\pm.70$ & $86.58\scriptstyle\pm.70$ & $\textbf{88.16}\scriptstyle\pm.90$
&$84.69\scriptstyle\pm.70$
& \cellcolor{gray!35}$\textbf{88.59}\scriptstyle\pm.60$   &
 $88.22\scriptstyle\pm.60$&$88.26\scriptstyle\pm.60$&$89.17\scriptstyle\pm.60$&$87.86\scriptstyle\pm.90$
 &$88.02\scriptstyle\pm.60$
 &\cellcolor{gray!35}$\textbf{90.01}\scriptstyle\pm.60$
 &\multirow{2}{*}{ $88.90\scriptstyle\pm.60$}\\
&$0.8$ & $77.11\scriptstyle\pm.80$&$78.32\scriptstyle\pm.80$ & $80.38\scriptstyle\pm.80$ & $\textbf{83.56}\scriptstyle\pm1.00$
&$79.96\scriptstyle\pm.80$
& \cellcolor{gray!35}$\textbf{83.69}\scriptstyle\pm.70$ &
  $87.58\scriptstyle\pm.61$& $88.02\scriptstyle\pm.60$& $88.01\scriptstyle\pm.60$ &$88.14\scriptstyle\pm.90$
  &$87.01\scriptstyle\pm.70$
  &  \cellcolor{gray!35}$\textbf{89.23}\scriptstyle\pm.60$ &\\
\midrule
\multirow{2}{*}{$0.9$}
& $0.5$&$80.11\scriptstyle\pm.80$&$81.06\scriptstyle\pm.80$ & $85.12\scriptstyle\pm.70$ &  \cellcolor{gray!35}$\textbf{88.10}\scriptstyle\pm2.00$
&$\textbf{86.68}\scriptstyle\pm.70$
& $\textbf{87.04}\scriptstyle\pm.70$&
 $88.38\scriptstyle\pm.60$& $87.79\scriptstyle\pm.60$&$\textbf{89.29}\scriptstyle\pm.60$&$88.32\scriptstyle\pm.70$
 &$\textbf{89.41}\scriptstyle\pm.60$
 &\cellcolor{gray!35}$\textbf{89.42}\scriptstyle\pm.60$&   \\
& $0.8$
& $72.87\scriptstyle\pm.90$&$73.60\scriptstyle\pm.90$ & $76.70\scriptstyle\pm1.10$ & \cellcolor{gray!35}$\textbf{82.10}\scriptstyle\pm2.40$
&$75.17\scriptstyle\pm.80$
& $\textbf{81.30}\scriptstyle\pm.80$ &
$\textbf{89.18}\scriptstyle\pm.60$&  \cellcolor{gray!35}$\textbf{89.22}\scriptstyle\pm.60$ & $88.12\scriptstyle\pm.80$ &$87.68\scriptstyle\pm.70$
&$88.50\scriptstyle\pm.60$
&$\textbf{88.77}\scriptstyle\pm.60$&\multirow{-2}{*}{$89.43{\scriptstyle\pm.60}$} \\
\bottomrule
\end{tabular}}
\vspace{2mm}
\caption{Table shows the test accuracy (\%), reported as mean $\pm$ 95\% CI of different methods on complete images and incomplete images with different missing probabilities on FashionMNIST. Shaded: best mean in each setting. Bold: the best result and results whose mean lies within its confidence interval.}
\label{tab:fashionmnist}
%\vspace{0.5mm}%\end{table}
\vspace{3mm}
%\begin{table}
\centering
\setlength\tabcolsep{4pt}
\renewcommand{\arraystretch}{1.1}
\resizebox{\textwidth}{!}{
\begin{tabular}{cc|>{\columncolor{gray!10}}c
>{\columncolor{gray!10}}c
>{\columncolor{gray!10}}c
>{\columncolor{gray!10}}c
>{\columncolor{gray!10}}c
>{\columncolor{gray!10}}c
|ccccccc}
\toprule
\multicolumn{2}{c|}{\textbf{Missing rate}} &
\multicolumn{6}{c|}{\cellcolor{gray!10}{\textbf{Incomplete data}}}&
\multicolumn{7}{c}{\textbf{Complete data}}
\\
\textbf{$P_\obs$} & \textbf{$P_\miss$}
& Zero & Mean & GAIN & SupMIWAE & LAC & IICN
& Zero & Mean & GAIN & SupMIWAE & LAC & IICN &\textit{CO}$^\dagger$ \\
\midrule
 
\multirow{2}{*}{$0.1$}
& $0.5$
& $61.65\scriptstyle\pm0.99$
& $61.52\scriptstyle\pm0.99$
& $67.60\scriptstyle\pm0.90$
& $69.06\scriptstyle\pm1.00$
& $72.36\scriptstyle\pm0.90$
& \cellcolor{gray!35}{$\textbf{75.92}\scriptstyle\pm1.10$}
& $69.44\scriptstyle\pm0.90$
& $69.48\scriptstyle\pm0.90$
& $72.01\scriptstyle\pm0.90$
& $73.00\scriptstyle\pm2.80$
& $73.30\scriptstyle\pm0.90$
& \cellcolor{gray!35}{$\textbf{78.98}\scriptstyle\pm1.10$}
& \multirow{2}{*}{\centering $59.15\scriptstyle\pm0.95$}
\\
 
& $0.8$
& $53.42\scriptstyle\pm1.00$
& $51.83\scriptstyle\pm1.00$
& $41.44\scriptstyle\pm1.00$
& $54.40\scriptstyle\pm1.00$
& $61.23\scriptstyle\pm1.00$
& \cellcolor{gray!35}{$\textbf{68.38}\scriptstyle\pm1.30$}
& $64.72\scriptstyle\pm0.83$
& $63.04\scriptstyle\pm0.91$
& $53.14\scriptstyle\pm1.00$
& $67.40\scriptstyle\pm2.90$
& $66.63\scriptstyle\pm0.90$
& \cellcolor{gray!35}{$\textbf{76.35}\scriptstyle\pm1.40$}
&
\\
 
\midrule
 
\multirow{2}{*}{$0.5$}
& $0.5$
& $58.87\scriptstyle\pm1.00$
& $59.31\scriptstyle\pm1.00$
& $63.68\scriptstyle\pm0.90$
& $66.18\scriptstyle\pm1.30$
& $73.67\scriptstyle\pm0.90$
& \cellcolor{gray!35}{$\textbf{74.77}\scriptstyle\pm0.90$}
& $69.44\scriptstyle\pm0.90$
& $73.23\scriptstyle\pm0.90$
& $73.68\scriptstyle\pm0.90$
& $76.78\scriptstyle\pm1.20$
& $76.72\scriptstyle\pm0.90$
& \cellcolor{gray!35}{$\textbf{78.27}\scriptstyle\pm0.80$}
& \multirow{2}{*}{\centering $75.50\scriptstyle\pm0.88$}
\\
 
& $0.8$
& $48.15\scriptstyle\pm1.00$
& $46.31\scriptstyle\pm1.00$
& $35.04\scriptstyle\pm0.90$
& $50.52\scriptstyle\pm1.40$
& $64.52\scriptstyle\pm1.00$
& \cellcolor{gray!35}{$\textbf{67.00}\scriptstyle\pm1.10$}
& $70.14\scriptstyle\pm0.90$
& $69.36\scriptstyle\pm0.91$
& $66.10\scriptstyle\pm0.90$
& $74.22\scriptstyle\pm1.20$
& $73.64\scriptstyle\pm0.90$
& \cellcolor{gray!35}{$\textbf{76.00}\scriptstyle\pm1.20$}
&
\\
 
\midrule
 
\multirow{2}{*}{$0.9$}
& $0.5$
& $47.98\scriptstyle\pm1.00$
& $49.08\scriptstyle\pm1.00$
& $49.36\scriptstyle\pm1.00$
& $61.10\scriptstyle\pm3.00$
& $65.05\scriptstyle\pm0.90$
& \cellcolor{gray!35}{$\textbf{76.00}\scriptstyle\pm1.00$}
& $76.26\scriptstyle\pm0.80$
& $76.80\scriptstyle\pm0.80$
& $74.55\scriptstyle\pm0.90$
& $77.01\scriptstyle\pm0.90$
& $75.68\scriptstyle\pm0.80$
& \cellcolor{gray!35}{$\textbf{78.48}\scriptstyle\pm1.10$}
& \multirow{2}{*}{\centering $78.00\scriptstyle\pm0.67$}
\\
& $0.8$
& $33.85\scriptstyle\pm0.90$
& $33.72\scriptstyle\pm0.90$
& $38.99\scriptstyle\pm1.00$
& $44.80\scriptstyle\pm3.10$
& $50.67\scriptstyle\pm1.00$
& \cellcolor{gray!35}{$\textbf{53.84}\scriptstyle\pm1.14$}
& $75.00\scriptstyle\pm0.83$
& $75.73\scriptstyle\pm0.80$
& $75.65\scriptstyle\pm0.80$
& $76.38\scriptstyle\pm0.90$
& $76.47\scriptstyle\pm0.80$
& \cellcolor{gray!35}{$\textbf{78.12}\scriptstyle\pm1.13$}
&
\\
\bottomrule
\end{tabular}}
\vspace{2mm}
\caption{Table shows the test accuracy (\%) , reported as mean $\pm$ 95\% CI of different methods on complete images and incomplete images with different missing probabilities on CIFAR-10. Shaded: best mean in each setting. There is no other result whose mean lies within the confidence interval of the best.}
\label{tab:cifar10}
%\vspace{mm}
%\end{table}
\vspace{3mm}
%\begin{table}[h]
\centering
\scriptsize
\setlength{\tabcolsep}{4pt}
\renewcommand{\arraystretch}{1.05}
\resizebox{1.0\textwidth}{!}{
\begin{tabular}{lcccccccccc}
\toprule
{\textbf{Dataset}} & \textbf{\#Feat.} & \textbf{\#Samp.} & \textbf{\#Comp.} & \textbf{$P_{\mathrm{miss}}$} &
\textit{CO}$^\dagger$ & Mean & GAIN & SupMIWAE & LAC & IICN \\
\midrule
 
\multirow{2}{*}{Mammographic Mass}
 & \multirow{2}{*}{5} & \multirow{2}{*}{961} & \multirow{2}{*}{831} & \multirow{2}{*}{0.2473}
 & $82.55\scriptstyle\pm0.104$
 & $82.66\scriptstyle\pm0.205$
 & $82.55\scriptstyle\pm0.201$
 & $81.56\scriptstyle\pm0.261$
 & $81.80\scriptstyle\pm0.147$
 & \cellcolor{gray!35}$\textbf{83.15}\scriptstyle\pm0.149$ \\
 & & & &
 & \cellcolor{gray!10}{--}
 & \cellcolor{gray!10}$80.91\scriptstyle\pm1.594$
 & \cellcolor{gray!10}$70.84\scriptstyle\pm1.773$
 & \cellcolor{gray!10}$78.01\scriptstyle\pm1.592$
 & \cellcolor{gray!10}$79.25\scriptstyle\pm1.643$
 & \cellcolor{gray!35}$\textbf{82.86}\scriptstyle\pm1.173$ \\[2pt]
 
\midrule
\multirow{2}{*}{Credit Approval}
 & \multirow{2}{*}{15} & \multirow{2}{*}{690} & \multirow{2}{*}{653} & \multirow{2}{*}{0.1207}
 & $82.55\scriptstyle\pm0.181$
 & $\textbf{85.90}\scriptstyle\pm0.288$
 & $85.00\scriptstyle\pm0.201$
 & $65.25\scriptstyle\pm0.589$
 & $\textbf{85.84}\scriptstyle\pm0.303$
 & \cellcolor{gray!35}$\textbf{85.92}\scriptstyle\pm0.218$ \\
 & & & &
 & \cellcolor{gray!10}{--}
 & \cellcolor{gray!10}$\textbf{74.77}\scriptstyle\pm6.004$
 & \cellcolor{gray!10}$59.22\scriptstyle\pm9.060$
 & \cellcolor{gray!10}$72.55\scriptstyle\pm7.118$
 & \cellcolor{gray!35}$\textbf{77.94}\scriptstyle\pm3.612$
 & \cellcolor{gray!10}$\textbf{75.72}\scriptstyle\pm4.771$ \\[2pt]
 
\midrule
\multirow{2}{*}{Diabetes}
 & \multirow{2}{*}{8} & \multirow{2}{*}{768} & \multirow{2}{*}{336} & \multirow{2}{*}{0.2207}
 & $75.37\scriptstyle\pm0.436$
 & $\textbf{77.84}\scriptstyle\pm0.365$
 & $77.20\scriptstyle\pm0.491$
 & $77.21\scriptstyle\pm0.472$
 & $77.28\scriptstyle\pm0.322$
 & \cellcolor{gray!35}$\textbf{78.11}\scriptstyle\pm0.563$ \\
 & & & &
 & \cellcolor{gray!10}{--}
 & \cellcolor{gray!10}$75.11\scriptstyle\pm0.302$
 & \cellcolor{gray!10}$74.95\scriptstyle\pm0.194$
 & \cellcolor{gray!10}$72.42\scriptstyle\pm0.362$
 & \cellcolor{gray!10}$74.52\scriptstyle\pm0.134$
 & \cellcolor{gray!35}$\textbf{77.26}\scriptstyle\pm0.223$ \\[2pt]
 
\midrule
\multirow{2}{*}{Hepatitis}
 & \multirow{2}{*}{19} & \multirow{2}{*}{155} & \multirow{2}{*}{80} & \multirow{2}{*}{0.1001}
 & $73.12\scriptstyle\pm2.585$
 & $\textbf{76.47}\scriptstyle\pm2.000$
 & $\textbf{75.12}\scriptstyle\pm2.379$
 & $72.76\scriptstyle\pm3.593$
 & $\textbf{77.12}\scriptstyle\pm2.320$
 & \cellcolor{gray!35}$\textbf{77.98}\scriptstyle\pm3.397$ \\
 & & & &
 & \cellcolor{gray!10}{--}
 & \cellcolor{gray!10}$60.00\scriptstyle\pm1.655$
 & \cellcolor{gray!10}$61.23\scriptstyle\pm1.874$
 & \cellcolor{gray!35}$\textbf{73.18}\scriptstyle\pm3.267$
 & \cellcolor{gray!10}$\textbf{71.59}\scriptstyle\pm2.052$
 & \cellcolor{gray!10}$68.36\scriptstyle\pm2.656$ \\[2pt]
 
\midrule
\multirow{2}{*}{Thyroid}
 & \multirow{2}{*}{21} & \multirow{2}{*}{3772} & \multirow{2}{*}{2642} & \multirow{2}{*}{0.0966}
 & $96.44\scriptstyle\pm0.008$
 & \cellcolor{gray!35}$\textbf{95.85}\scriptstyle\pm0.037$
 & $79.52\scriptstyle\pm1.468$
 & $92.00\scriptstyle\pm0.010$
 & $95.11\scriptstyle\pm0.031$
 & $95.53\scriptstyle\pm0.013$ \\
 & & & &
 & \cellcolor{gray!10}{--}
 & \cellcolor{gray!10}$94.86\scriptstyle\pm0.352$
 & \cellcolor{gray!10}$86.65\scriptstyle\pm0.973$
 & \cellcolor{gray!10}$97.13\scriptstyle\pm0.013$
 & \cellcolor{gray!10}$98.31\scriptstyle\pm0.032$
 & \cellcolor{gray!35}$\textbf{98.38}\scriptstyle\pm0.012$ \\[2pt]
 
\bottomrule
\end{tabular}}
\vspace{2mm}%\midrule
%\multirow{2}{*}{Nomao} 
%& \multirow{2}{*}{83} & \multirow{2}{*}{34465} & \multirow{2}{*}{9159} & \multirow{2}{*}{0.1172} & $78.12\scriptstyle\pm11.32$
% & $74.12\scriptstyle\pm15.33$ & $70.50\scriptstyle\pm11.54$ & $63.23\scriptstyle\pm15.21$ && \cellcolor{gray!35}$\textbf{75.27}\scriptstyle\pm11.94$ \\
% &  &  &  &  &\cellcolor{gray!10}{--} &\cellcolor{gray!10}$70.28\scriptstyle\pm07.42$ & \cellcolor{gray!10}$60.21\scriptstyle\pm04.98$ & \cellcolor{gray!10}$55.30\scriptstyle\pm04.42$& \cellcolor{gray!10}& \cellcolor{gray!35}$\textbf{70.39}\scriptstyle\pm12.07$ \\[2pt]
 
\caption{
Test accuracy (\%), reported as mean $\pm$ standard deviation
over $5$-fold cross-validation, of different methods on complete~(first row)
and incomplete~(second row) data across UCI datasets with naturally occurring
missingness. Shaded cell:
best mean in each row. Bold: the best result and results whose mean lies
within one standard deviation of the best.  \textcolor{black}{\textbf{\#Feat.} is the number of features. \textbf{\#Samp.} is the number of training samples. \textbf{\#Comp.} is the number of complete samples in the training set.}}
\label{tab:uci}
%\vspace{-6mm}
\end{table}

\paragraph{Real-world missing data} %We further evaluate all methods on the UCI benchmark datasets with natural missing values. Unlike artificially masked data, these datasets contain real-world missingness caused by uncontrolled data collection, where both the distribution and mechanism of missingness are unknown. As these datasets are relatively small, training deep models tends to suffer from high variance, overfitting, and unstable generalization\citep{vapnik1998statistica}, which reflect the classic bias–variance trade-off in small-sample regimes\citep{hastie2009elements}. Consistent with this, deep networks trained on limited data often exhibit unstable generalization and sensitivity to initialization \citep{zhang2017understanding, guo2017calibration}. Prior studies have further shown that ensemble and Bayesian approaches can provide more stable and better-calibrated predictions under such data scarcity and distributional uncertainty \citep{maddox2019simple}. To mitigate this issue, we adopt a cross-validation ensemble which effectively reduces variance and stabilizes model performance in limited-data regimes.
%Beyond variance reduction on the classifier side, we also introduce a multiple-imputation strategy specifically for datasets that are simultaneously small and exhibit higher levels of missingness. In such settings, the generator receives limited supervision on the observed covariates and faces significant uncertainty when inferring the conditional distribution over missing dimensions. 

The UCI tabular datasets contain naturally occurring missing values. Since the underlying missingness mechanism is unknown, this setting better reflects realistic missing-data scenarios.
Table~\ref{tab:uci} shows that IICN effectively narrows the performance gap between complete and imputed inputs, and can even outperform the complete-only baseline. 

\textcolor{black}{The experimental results suggest that IICN offers larger improvements in settings where simple imputation, such as zero or mean imputation, performs poorly. When simple imputation already yields strong classification performance, the gains from IICN are more modest. These results suggest that IICN is particularly useful when missingness poses a substantial challenge to classification, while simpler imputation-and-classification pipelines may suffice in easier settings.}

\section{Conclusion and limitations}\label{sec:limit}
%Classification on missing data is a fundamental problem, as data loss or corruption frequently occurs in real-world applications. This scenario calls for a reliable supervised learning algorithm that can robustly handle varying degrees of missingness and predict correctly. We propose IICN, a novel model for missing data classification, which utilizes the adversarial framework to ensure that imputations are aligned with predictive performance, enabling the classifier to leverage more beneficial imputations. Experiments on both synthetic and real-world datasets demonstrate that IICN consistently outperforms state-of-the-art Impute-then-Classify methods in terms of classification accuracy and robustness across different levels of missingness. As shown in Table~\ref{tab:uci}, IICN consistently outperforms other methods across all datasets, achieving the highest classification accuracy, especially on datasets with moderate to high levels of missingness (e.g., Hepatitis and Mammographic Mass). The performance gap between IICN and Mean imputation is particularly large (up to +25$\%$), highlighting the importance of learning imputations that are aligned with the classification objective.
Classification with missing data is a fundamental challenge in real-world applications. We propose IICN, a task-aligned adversarial framework that learns imputations jointly with the downstream classification objective, enabling effective use of incomplete samples without degrading performance on fully observed data. Experiments on synthetic, image, and tabular datasets show that IICN consistently outperforms impute-then-classify and other joint optimization baselines across a wide range of missingness levels. These results highlight the importance of task-aligned imputation for robust classification under missing data.

\textcolor{black}{Our method has several limitations. First, the existence of the theoretical minimizer requires a non-zero probability of observing complete samples, i.e., $\mathbb{P}(R=1)>0$, which may not hold in some practical settings. Second, our Bayes-optimality guarantee relies on the missing completely at random (MCAR) assumption and does not directly extend to more general missingness mechanisms. Third, our method does not explicitly model discrete or categorical features; instead, these features are encoded and treated as continuous inputs, following the same preprocessing as the baselines. Developing architectures that natively support mixed-type data is an important direction for future work. Finally, the scalability of our method to datasets with a large number of classes remains to be further investigated.}
\begin{ack}
This work was funded in part by Telus Communications and the Natural Sciences and Engineering Research Council
of Canada.
\end{ack}
\newpage
\bibliographystyle{configuration/iclr2021_conference.bst}
\bibliography{reference}
\newpage
\appendix
\newpage
\onecolumn
\section{Related work}\label{sec:related}
\paragraph{Classification under missing data.}
Classical approaches to missing data classification follow the \emph{impute-then-classify} pipeline, where missing values are first filled in via simple statistics (zero, mean, mode) or learned imputation models, and a classifier is then trained on the completed data~\citep{little1995statistical, vanbuuren2018flexible}. Modern imputation methods leverage deep generative models such as VAEs~\citep{mattei2019miwae, ipsen2021not} and GANs~\citep{gain} to capture complex feature distributions. However, treating imputation and classification as separate stages can produce imputations that are distributionally faithful but not necessarily classification-friendly~\citep{paterakis2024we, morvan2025imputation, josse2024consistency}. A line of recent work addresses this mismatch by jointly training imputers and classifiers, including GACN~\citep{gacn}, SupMIWAE~\citep{SupMIWAE}, and LAC~\citep{sugiyama2025discrete, sugiyama2025unified}, which  use varying combinations of adversarial losses, variational inference, and weighted reconstruction-classification objectives. IICN belongs to this prediction-aligned paradigm but introduces a class-aware adversarial formulation that is likelihood-free and provably recovers the Bayes-optimal classifier under MCAR.

\paragraph{Generative imputation models.}
Generative models for missing data have evolved along two main categories. Variational approaches model $p(\xx_{\miss}\mid \xx_{\obs})$ through an encoder-decoder architecture trained with the IWAE bound~\citep{mattei2019miwae, ipsen2021not}. SupMIWAE extends this framework by jointly training a classifier with the variational posterior, enabling importance-weighted marginalization at inference. Adversarial approaches train a generator to produce imputations indistinguishable from the real data distribution. In particular, GAIN uses a mask-aware discriminator~\citep{gain}, whereas MisGAN explicitly models both the data distribution and the missingness mechanism~\citep{li2018learning}. However, their training objectives are primarily defined in terms of distributional reconstruction. These methods optimize imputation quality without considering the downstream task, and they do not directly optimize the decision boundary of the downstream classifier.  IICN addresses this gap by combining adversarial imputation with task-aware discrimination in a single $(n+1)$-class objective, where the extra class captures imputation artifacts while the remaining classes estimate the posterior label distribution over completed inputs.

\paragraph{Semi-supervised GANs.}
The $(n{+}1)$-class adversarial formulation in IICN is inspired by semi-supervised GAN training~\citep{salimans2016improved, springenberg2015unsupervised, dai2017good}. In these methods, the unsupervised adversarial loss is consistent with the supervised classification loss, allowing labelled, unlabeled, and generated samples to jointly constrain the classifier. IICN uses an analogous output structure for a different purpose: the additional class represents incomplete or improperly imputed inputs, and the generator is trained to complete missing features from observed features rather than synthesize full samples from noise. In this sense, \cite{salimans2016improved} uses the generator to make semi-supervised classification more effective, whereas IICN uses the classifier-discriminator structure to make the imputation task-aware.

\paragraph{Theoretical analysis of classification under missingness.}
Recent studies have begun to characterize what the optimal predictor is in the presence of a missing value~\citep{bayes_optimal, neumiss, josse2024consistency, sugiyama2025discrete, sugiyama2025unified}. \citet{bayes_optimal} characterize the Bayes-optimal predictor as a marginalization over the conditional distribution $p(\xx_{\miss}\mid \xx_{\obs})$, and \citet{neumiss} approximate this predictor through a Neumann-series expansion under linear-Gaussian assumptions. \citet{sugiyama2025discrete, sugiyama2025unified} establish 
asymptotic risk bounds on the LAC training objective, but their guarantees apply only to the empirical loss rather than to the distance from the Bayes-optimal classifier. IICN provides a stronger theoretical result. Namely, at the adversarial equilibrium, the composition of the optimal generator and discriminator \emph{exactly} recovers the Bayes-optimal predictive rule under MCAR, without requiring distributional assumptions or variational approximations.

\section{Proofs}\label{sec:proof}
\subsection{Technical tools}
\begin{claim}\label{claim: max}
Consider a set of $n$-variables, $\xx=(\xx_1,...,\xx_n)\in(\mathbb{R}^{+})^n$, and the following constrained objective function $f(\xx):\ (\mathbb{R}^{+})^n\rightarrow \mathbb{R}$:$$f(\xx) = a_1\log \xx_1+...+a_n\log \xx_n,\ s.t\ \sum_i^n\xx_i=1,$$ where $a_i\in[0,1]$ and $\sum a_i=1.$ The maximum value is obtained when $\xx_i=a_i,\ \forall i\in[n]$.
 \end{claim}
 \begin{proof}
 	Introducing Lagrangian multiplier $\lambda$ and we get Lagrangian function $$\cL(\xx, \lambda) = \sum_i^na_i\log \xx_i + \lambda\sum_i^n\xx_i.$$
	Consider the partial derivative of $\cL(\xx, \lambda)$ w.r.t $\xx_i$, 
			\begin{equation}	\nonumber
				\begin{aligned}
					\frac{\partial \cL(\xx, \lambda)}{ \partial \xx_i} &= \frac{a_i}{\xx_i} + \lambda =0	\\
					\lambda & = -\frac{a_i}{\xx_i},\ \forall i.
				\end{aligned}
			\end{equation}
			This implies that $\forall i,j\in[n]$, $\frac{a_i}{\xx_i} = \frac{a_j}{\xx_j}$. Then we w.l.o.g pick up the $j$-th variable, and plug it in the equality constraint to solve for $x_j^{\star}$:
			\begin{equation}
				\begin{aligned}
					\sum_s\xx_s &= \xx_j + \sum_{s\neq j}\xx_s	\\
					& = \xx_j + \sum_{s\neq j} \frac{a_s}{a_j}\xx_j	\\
					& = \frac{\xx_j}{a_j}(\sum_sa_s) = \frac{\xx_j}{a_j} = 1.
				\end{aligned}
			\end{equation}
			This implies that $\xx_j^{\star} = a_j$.
		Consider the second partial derivative of $\cL(\xx, \lambda)$ w.r.t $\xx^{\star}$, $$\frac{\partial^2 \cL(\xx, \lambda)}{\partial \xx_i^2} = -\frac{a_i}{\xx_i^2}\leq 0.$$ The second derivative of any $\xx$ in its domain is always non-positive with $a_i\geq 0$. Then we can conclude the stationary point $\xx^{\star}$ is indeed the maximum point. And it's globally and uniquely maximal if $a_i>0$.
	 \end{proof}
 \subsection{Proof of Thm.~\ref{thm:ds}}\label{sec:opt_d}
% \begin{proposition}The unique optimum set of $\{D_i\}$ is of form, 
%\begin{equation}\nonumber
%	\begin{aligned}
%		D_i^{\star}(\hat\xx) &= \frac{p(Y=i, \hat\xx, R=1)}{\sum_jp(Y=j, \hat\xx\mid R=1)\mP(R=1) + p(\hat\xx\mid R=0)\mP(R=0)} 	\\
%		& = \mP(Y=i, R=1\mid \hat\xx).	\end{aligned}
%\end{equation} 
%\end{proposition}
 \begin{proof}
Consider the explicit formulation of the maximum problem w.r.t only $\{D_i\}$,
	\begin{equation}\nonumber
		\begin{aligned}
			V(\{D_i\})  & =\mP(R=1)\sum_i^n\mP(Y=i)\int p(\hat\xx\mid Y=i, R=1) \log D_i(\hat\xx)\, \mathrm{d}\hat\xx \\&\quad\  + \mP(R=0)\int p(\hat\xx\mid R=0) \log \Big(1-\sum_jD_j(\hat\xx)\Big)\, \mathrm{d}\hat\xx \\
				& =\int \sum_i^np(\hat\xx,Y=i, R=1) \log D_i(\hat\xx)+p(\hat\xx,R=0) \log \Big(1-\sum_jD_j(\hat\xx)\Big)\, \mathrm{d}\hat\xx \\
		\end{aligned}
	\end{equation}
By Claim~\ref{claim: max},  problem $\max_{\substack{b_1,\ldots,b_n \\ \sum_{i=1}^n b_i = 1}}\sum_{i=1}^n a_i \log b_i$ uniquely reaches the optimal when $b_i^{\star}=\frac{a_i}{\sum_ja_j}$. Then $V(\{D_i\})$ reaches the unique optimum when $$D_i^{\star}(\hat\xx) = \frac{p(\hat\xx, i,R=1)}{\sum_i^np(\hat\xx, i,R=1) + p(\hat\xx,R=0)} = \mP(Y=i, R=1\mid \hat\xx),$$ 
 which concludes the proof.
\end{proof}
\subsection{Proof of Thm.~\ref{thm:generator}}\label{sec:opt_g}
%\begin{theorem}
%	The global minimum of  $V(G, \{D_i^{\star}\})$ is uniquely achieved at the point when $p(\hat\xx\mid R=1)=p(\hat\xx\mid R=0)$.
%\end{theorem}
%\begin{assumption}[Data domain]It's assumed that the data space for original data and generated data are identical, and we denote it as $\chi\subset\mathbb{R}^d$.
%\end{assumption}
\begin{proof}
\begin{equation}\nonumber
	\begin{aligned}
		V(G, \{D_i\}^{\star}) & = \mP(R=1)\int \sum_i^np(\hat\xx,Y=i\mid R=1) \log D^{\star}_i(\hat\xx)\, \mathrm{d}\xx 	\\
		&\quad\quad+ \mP(R=0)\int p(\hat\xx\mid R=0) \log \bigg(1-\sum_i^nD^{\star}_i(\hat\xx)\bigg)\, \mathrm{d}\xx	\\
		& = \int \sum_i^np(\hat\xx,Y=i,R=1) \log \frac{p(\hat\xx,Y=i, R=1)}{p(\hat\xx)} + p(\hat\xx, R=0) \log \frac{p(\hat\xx, R=0)}{p(\hat\xx)}\, \mathrm{d}\xx	\\
		& = \int \sum_i^np(\hat\xx,Y=i,R=1) \log \frac{\mP(R=1)\mP(Y=i\mid \hat\xx, R=1)p(\hat\xx\mid R=1)}{p(\hat\xx)} 	\\
		&\quad\quad+ p(\hat\xx, R=0) \log \frac{p(\hat\xx\mid R=0)\mP(R=0)}{p(\hat\xx)}\, \mathrm{d}\xx		\\
		& = \mP(R=1)\int \underbrace{\sum_i^np(Y=i, \hat\xx\mid R=1)}_{=p(\hat\xx\mid R=1)} \log \frac{p(\hat\xx\mid R=1)}{p(\hat\xx)}\, \mathrm{d}\xx 	\\
		&\quad\quad+ \mP(R=0)\int p(\hat\xx\mid R=0) \log \frac{p(\hat\xx\mid R=0)}{p(\hat\xx)}\, \mathrm{d}\xx		\\
		& \quad\quad+ \int p(\hat\xx,R=1)\sum_i^n\mP(Y=i\mid \hat\xx,R=1) \log \ \mP(Y=i\mid \hat\xx, R=1)\, \mathrm{d}\xx  	\\
		& \quad\quad+ \int \mP(R=1)\sum_i^np(Y=i, \hat\xx\mid R=1) \log \ \mP(R=1)\, \mathrm{d}\xx 	\\
		& \quad\quad+ \mP(R=0)\int p(\hat\xx\mid R=0) \log \ \mP(R=0)\, \mathrm{d}\xx	\\
		& =\mP(R=1)D_{\text{KL}}(\hat\xx\mid R=1\,\|\ \hat\xx) + \mP(R=0)D_{\text{KL}}(\hat\xx\mid R=0\,\|\ \hat\xx) 	\\
		&\quad\quad+ \underbrace{\mP(R=1)\int p(\hat\xx\mid R=1)\sum_i^n\mP(Y=i\mid \hat\xx,R=1) \log \mP(Y=i\mid \hat
		\xx,R=1)\, \mathrm{d}\xx}_{=-H(Y\mid \hat\xx, R=1)}	\\
		&\quad\quad+\underbrace{\mP(R=1) \log \ \mP(R=1) + \mP(R=0) \log \ \mP(R=0)}_{=-H_{\text{binomial}}(R)}
	\end{aligned}
\end{equation}
Since we know that $H_{\text{binomial}}(R)$ and $H(Y\mid \hat\xx, R=1)$ are constant w.r.t generator $G$.\footnote{Note that $H_{\text{binomial}}(R)$ is not the entropy for arbitrary missing pattern, but for indicator variable which only contains two outcomes, $R=1$ and $R=0$. Furthermore, since $P(Y\mid \hat\xx, R=1)$ and $p(\hat\xx\mid R=1)$ are irrelevant to the generator $G$, the entropy $H(Y\mid \hat\xx, R=1)$ is constant w.r.t. $G$.}.Then due to the optimality criterion of KL-divergence, the global minimum $V(G^{\star}, D^{\star})$ is obtained at $p(\hat\xx^{\star}\mid R=1) = p(\hat\xx^{\star}) = p(\hat\xx^{\star}\mid R=0)$. {Then equality holds since for each missing pattern $\mm$, we have $p(\xx_\obs\mid R=1)=p(\xx_\obs\mid R=0)=p(\xx_\obs)$ under MCAR.}
\end{proof}

\subsection{Proof of Thm.~\ref{thm:classifier}}\label{sec:opt_c}

We prove the recovery of the Bayes-optimal classifier by composing 
the optimality results for $D^\star$ and $G^\star$ established in 
Thm.~\ref{thm:ds} and~\ref{thm:generator}.
\begin{proof}
Consider an incomplete observation $\tilde\xx_0$ with missing 
mask $\mathbf{m}_0\neq \unit$, which partitions the coordinates into observed indices $\obs:=\{j:(\mm_0)_j=1\}$ and missing indices $\miss := \obs^c$. Then the corresponding imputation for this observation is $\hat\mX_{\miss}^\star = [G^\star(\tilde\xx_0, \mathbf{m}_0, \mZ)]_{\miss}$ with $\mZ \sim p_{\mZ}(\zz)$.

By Eq.~\eqref{eq:normalized}, for any $\hat\xx \in \mathbb{R}^d$ 
the optimal discriminator satisfies
\begin{equation}
\frac{D_i^\star(\hat\xx)}{\sum_{j=1}^n D_j^\star(\hat\xx)} 
= \mP(Y=i \mid \mX = \hat\xx).
\end{equation}
In particular, evaluating at $\hat\xx = (\xx_{\obs}, \hat\xx_{\miss}^\star)$ 
where $\hat\xx_{\miss}^\star$ is any realization of 
$\hat\mX_{\miss}^\star$, 
we obtain
\begin{equation}\label{eq:proof_pointwise}
\frac{D_i^\star(\xx_{\obs}, \hat\xx_{\miss}^\star)}
     {\sum_{j=1}^n D_j^\star(\xx_{\obs}, \hat\xx_{\miss}^\star)}
= \mP(Y=i \mid \mX = (\xx_{\obs}, \hat\xx_{\miss}^\star)).
\end{equation}

{
By the distribution alignment property in
Cor.~\ref{cor:match}, the completed random variable
produced by $G^\star$ satisfies
$\hat\mX^\star \stackrel{d}{=} \mX$.
Furthermore, under Condition~\ref{cond:independence},
$\hat\mX^\star$ is independent of the missingness
pattern $\mM$. Therefore, for any missingness pattern
$\mathbf{m}_0$ with positive probability,
\begin{equation}
p_{\hat\mX^\star}(\uu \mid \mM=\mathbf{m}_0)
= p_{\hat\mX^\star}(\uu)
= p_{\mX}(\uu).
\end{equation}
By construction,
$\hat\mX^\star =
\tilde\mX\odot\mM +
G^\star(\tilde\mX,\mM,\mZ)\odot(\unit-\mM)$,
so the observed coordinates are preserved, i.e.,
$\hat\mX^\star\odot\mM=\mX\odot\mM$.
Consequently, conditioning on the observed values
$\xx_{\obs}$ under the fixed pattern $\mathbf{m}_0$
yields
\begin{equation}\label{eq:proof_dist_match}
\hat\mX_{\miss}^\star
\mid (\xx_{\obs},\mM=\mathbf{m}_0)
\stackrel{d}{=}
\mX_{\miss}\mid\xx_{\obs}.
\end{equation}
Thus, the imputed missing coordinates follow the
true conditional distribution
$p_{\mX_{\miss}}(\xx_{\miss}\mid\xx_{\obs})$.}
Concretely, the induced density of $\hat\mX_{\miss}^\star$ given 
$\xx_{\obs}$ takes the form
\begin{equation}\nonumber%\label{eq:proof_g_density}
p_{\hat\mX_{\miss}^\star}( \uu\mid \xx_{\obs}) 
= \int \delta\!\left(\uu - [G^\star(\tilde\xx_0, 
\mm_0, \zz)]_{\miss}\right) p_{\mZ}(\zz) \, \mathrm{d}\zz 
= p_{\mX_{\miss}}(\uu \mid \xx_{\obs}),\quad\uu\in\mathbb{R}^{\mid \miss\mid }
\end{equation}
i.e., random sampling of $\zz \sim p_{\mZ}(\zz)$ followed by 
$G^\star$'s transformation produces samples from the true 
conditional distribution.

Then taking expectation of Eq.~\eqref{eq:proof_pointwise} with respect to $\hat\mX_{\miss}^\star \mid \xx_{\obs}$, and applying Eq.~\eqref{eq:proof_dist_match}, we have
\begin{equation}
\begin{aligned}
\mathbb{E}_{\hat\mX_{\miss}^\star \mid \xx_{\obs}}\!\left[
\frac{D_i^\star(\xx_{\obs}, \hat\mX_{\miss}^\star)}
     {\sum_{j=1}^n D_j^\star(\xx_{\obs}, \hat\mX_{\miss}^\star)}
\right]
&\stackrel{(\ref{eq:pointwise})}{=} \int \mP(Y=i \mid \mX = (\xx_{\obs}, \uu))\,
       p_{\hat\mX_{\miss}^\star}(\uu \mid \xx_{\obs})\, \mathrm{d}\uu\\
&\stackrel{(\ref{eq:proof_dist_match})}{=} \int \mP(Y=i \mid \mX = (\xx_{\obs}, \uu))\,
       p_{\mX_{\miss}}(\uu \mid \xx_{\obs})\, \mathrm{d}\uu\\
& = \mP(Y=i \mid \xx_{\obs}).\\
\end{aligned}
\end{equation}
This is the Bayes-optimal classifier under missingness, completing 
the proof. \end{proof}

\newpage
\section{Experiments}\label{app:experiment}
All experiments were conducted on the NVIDIA H100 80GB HBM3 GPUs (CUDA 12.8). Each job was executed on a single GPU worker with approximately 80GB GPU memory available. %Adam is used for experiments on real-world datasets, while SGD is used for experiments on synthetic datasets to reduce overfitting.
\subsection{Experimental details}\label{app:exp}
Across all experiments in this paper, we deliberately adopt a simple and consistent training scheme: On the CIFAR-10, FashionMNIST, and UCI datasets, we use Adam with a constant learning rate, no weight decay, and no data augmentation beyond basic normalization. On the synthetic Burger dataset, we use SGD with a constant learning rate
and no weight decay. Then for fair comparison, we apply this exact same training scheme to obtain the complete-data upper bound. Using the same classifier backbone architectures as in our missing-data experiments (LeNet for FashionMNIST and ResNet-8 for CIFAR-10), we obtain $91\%$ accuracy on FashionMNIST and $78.15\%$ accuracy on CIFAR-10 on complete data. We note this CIFAR-10 result is lower than typical CIFAR-10+ResNet benchmarks (often $84\%$), which rely on data augmentation and learning-rate scheduling. This gap is expected given our simplified and consistent scheme. Note that a similar complete-data upper bound is not well-defined for the tabular datasets used in this paper, as they exhibit missingness naturally rather than having missingness synthetically introduced onto an originally complete dataset. On the synthetic Burger dataset, we can achieve $99.30\%$ accuracy using the same training scheme as in missing-data experiments.

\subsection{Model architecture}\label{app:model}
Table~\ref{tab:gen-arch} shows the backbone architecture for generator, and Table~\ref{tab:miwae-arch} shows the backbone architecture for encoder/decoder in MIWAE.
\begin{table}[h]
\centering
\small
\setlength{\tabcolsep}{6pt}
\renewcommand{\arraystretch}{1.05}
%\resizebox{\textwidth}{!}{
\begin{tabular}{c l c c}
\toprule
\multicolumn{4}{c}{\textbf{Generator} (shared; $n_{\text{ch}}=1$ for FashionMNIST, $n_{\text{ch}}=3$ for CIFAR-10)} \\
\hline
\textbf{Layer} & \textbf{Operation} & \textbf{Kernel / Stride / Padding} & \textbf{Out Ch.} \\
\hline

\multirow{3}{*}{1} & Conv2D    & $3\times3\ /\ 1\ /\ 1$ & 128 \\
                   & BatchNorm & --                    & 128 \\
                   & ReLU      & --                    & 128 \\
\hline

\multirow{3}{*}{2} & Conv2D    & $3\times3\ /\ 1\ /\ 1$ & 64 \\
                   & BatchNorm & --                    & 64 \\
                   & ReLU      & --                    & 64 \\
\hline

3 & Conv2D & $3\times3\ /\ 1\ /\ 1$ & $n_{\text{ch}}$ \\
4 & Tanh   & --                    & $n_{\text{ch}}$ \\
\bottomrule
\end{tabular}%}
\vspace{1mm}
\caption{Generator architecture. Input to layer 1 has $2n_{\text{ch}}$ channels (FashionMNIST: 2, CIFAR-10: 6).}
\label{tab:gen-arch}
%\end{table}

%\begin{table}[t]
\centering
\small
\setlength{\tabcolsep}{6pt}
\renewcommand{\arraystretch}{1.05}
\resizebox{\textwidth}{!}{
\begin{tabular}{c l c c | c l c c}
\toprule 
\multicolumn{4}{c}{\textbf{FashionMNIST MIWAE} (Gray, 1 ch.)} &
\multicolumn{4}{c}{\textbf{CIFAR-10 MIWAE} (RGB, 3 ch.)} \\
\hline
\textbf{Layer} & \textbf{Operation} & \textbf{Kernel / Stride / Padding} & \textbf{Out Ch.} &
\textbf{Layer} & \textbf{Operation} & \textbf{Kernel / Stride / Padding} & \textbf{Out Ch.} \\
\hline

\multicolumn{4}{l|}{\textbf{Encoder}} &
\multicolumn{4}{l}{\textbf{Encoder}} \\
\hline

\multirow{3}{*}{1} & Conv2D     & $3\times3\ /\ 1\ /\ 1$ & 64  &
\multirow{3}{*}{1} & Conv2D     & $3\times3\ /\ 1\ /\ 1$ & 64  \\
                   & BatchNorm  & --                    & 64  &
                   & BatchNorm  & --                    & 64  \\
                   & ReLU       & --                    & 64  &
                   & ReLU       & --                    & 64  \\
\hline

\multirow{3}{*}{2} & Conv2D     & $3\times3\ /\ 2\ /\ 1$ & 128 &
\multirow{3}{*}{2} & Conv2D     & $3\times3\ /\ 2\ /\ 1$ & 128 \\
                   & BatchNorm  & --                    & 128 &
                   & BatchNorm  & --                    & 128 \\
                   & ReLU       & --                    & 128 &
                   & ReLU       & --                    & 128 \\
\hline

\multirow{3}{*}{3} & Conv2D     & $3\times3\ /\ 2\ /\ 1$ & 128 &
\multirow{3}{*}{3} & Conv2D     & $3\times3\ /\ 2\ /\ 1$ & 128 \\
                   & BatchNorm  & --                    & 128 &
                   & BatchNorm  & --                    & 128 \\
                   & ReLU       & --                    & 128 &
                   & ReLU       & --                    & 128 \\
\hline

\multirow{3}{*}{4} & Conv2D     & $3\times3\ /\ 2\ /\ 1$ & 128 &
\multirow{3}{*}{4} & Conv2D     & $3\times3\ /\ 2\ /\ 1$ & 128 \\
                   & BatchNorm  & --                    & 128 &
                   & BatchNorm  & --                    & 128 \\
                   & ReLU       & --                    & 128 &
                   & ReLU       & --                    & 128 \\
\hline

5 & Linear ($\mu$)        & $2048 \rightarrow 20$ & 20 &
5 & Linear ($\mu$)        & $2048 \rightarrow 20$ & 20 \\
6 & Linear ($\log\sigma^2$) & $2048 \rightarrow 20$ & 20 &
6 & Linear ($\log\sigma^2$) & $2048 \rightarrow 20$ & 20 \\
\hline

\multicolumn{4}{l|}{\textbf{Decoder}} &
\multicolumn{4}{l}{\textbf{Decoder}} \\
\hline

1 & Linear & $20 \rightarrow 2048$ & 2048 &
1 & Linear & $20 \rightarrow 2048$ & 2048 \\
\hline

\multirow{3}{*}{2} & Deconv (ConvT2D) & $3\times3\ /\ 2\ /\ 1$ & 128 &
\multirow{3}{*}{2} & Deconv (ConvT2D) & $4\times4\ /\ 2\ /\ 1$ & 128 \\
                   & BatchNorm         & --                    & 128 &
                   & BatchNorm         & --                    & 128 \\
                   & ReLU              & --                    & 128 &
                   & ReLU              & --                    & 128 \\
\hline

\multirow{3}{*}{3} & Deconv (ConvT2D) & $4\times4\ /\ 2\ /\ 1$ & 128 &
\multirow{3}{*}{3} & Deconv (ConvT2D) & $4\times4\ /\ 2\ /\ 1$ & 128 \\
                   & BatchNorm         & --                    & 128 &
                   & BatchNorm         & --                    & 128 \\
                   & ReLU              & --                    & 128 &
                   & ReLU              & --                    & 128 \\
\hline

\multirow{3}{*}{4} & Deconv (ConvT2D) & $4\times4\ /\ 2\ /\ 1$ & 64  &
\multirow{3}{*}{4} & Deconv (ConvT2D) & $4\times4\ /\ 2\ /\ 1$ & 64  \\
                   & BatchNorm         & --                    & 64  &
                   & BatchNorm         & --                    & 64  \\
                   & ReLU              & --                    & 64  &
                   & ReLU              & --                    & 64  \\
\hline

5 & Conv2D ($\mu_x$)        & $3\times3\ /\ 1\ /\ 1$ & 1 &
5 & Conv2D ($\mu_x$)        & $3\times3\ /\ 1\ /\ 1$ & 3 \\
6 & Conv2D ($\log\sigma_x^2$) & $3\times3\ /\ 1\ /\ 1$ & 1 &
6 & Conv2D ($\log\sigma_x^2$) & $3\times3\ /\ 1\ /\ 1$ & 3 \\
7 & Tanh & -- & 1 &
7 & Tanh & -- & 3 \\
\bottomrule
\end{tabular}}
\vspace{1mm}
\caption{MIWAE architectures for FashionMNIST and CIFAR-10}
\label{tab:miwae-arch}
\end{table}
\subsection{Inference via Monte Carlo marginalization.}\label{app:inference}
At test time, we approximate the marginalization in Thm.~\ref{thm:classifier} by Monte Carlo averaging over $K$ imputations. Specifically, we draw $\{\zz^{(k)}\}_{k=1}^{K}\sim P_{\mZ}$, form imputations $\hat{\xx}^{(k)}_{\miss}=[G_{\theta}(\tilde\xx,\mm,\zz^{(k)})]_{\miss}$, and compute
 {
\setlength{\abovedisplayskip}{2pt}
\setlength{\belowdisplayskip}{2pt}
\setlength{\abovedisplayshortskip}{2pt}
\setlength{\belowdisplayshortskip}{2pt}
\[
\mP(y=i \mid \xx_{\obs})
\;\approx\;
\frac{1}{K}\sum_{k=1}^{K}
\frac{(D_{\boldsymbol{\omega}})_i\!\left(\xx_{\obs}, \hat{\xx}^{(k)}_{\miss}\right)}
{\sum_{j=1}^{n} (D_{\boldsymbol{\omega}})_j\!\left(\xx_{\obs}, \hat{\xx}^{(k)}_{\miss}\right)}.
\]}
\subsection{Further comparison}
We further consider additional comparisons with HyperImpute~\citep{hyperimpute}, PC-GAIN~\citep{pc-gain}, GACN~\citep{gacn}, and 1-NN~\citep{1nn} on the synthetic Burger dataset (noise level $0.4$) in Table~\ref{tab:burger_more}, FashionMNIST in Table~\ref{tab:fashionmnist_more}, CIFAR-10 in Table~\ref{tab:cifar10_more}, and UCI tabular datasets in Table~\ref{tab:uci_more}, respectively. HyperImpute is omitted from the image experiments because its per-column model search is computationally prohibitive at these input dimensions. On the Burger dataset, HyperImpute, PC-GAIN, and GACN attain majority-class accuracy ($66.7\%$) on both complete and incomplete test samples at a missing rate of $0.7$, while IICN maintains high accuracy on complete samples. On the image benchmarks, IICN and GACN generally outperform PC-GAIN and 1-NN; their performance is comparable on FashionMNIST, whereas IICN achieves higher mean accuracy in most CIFAR-10 settings. IICN achieves the highest mean accuracy in the most evaluated UCI datasets, although some competing results fall within its reported confidence intervals.
\begin{table}[h]
\centering
\setlength{\tabcolsep}{4pt}
\renewcommand{\arraystretch}{1.1}
\resizebox{\textwidth}{!}{
\begin{tabular}{c|
>{\columncolor{gray!10}}c
>{\columncolor{gray!10}}c
>{\columncolor{gray!10}}c
>{\columncolor{gray!10}}c
>{\columncolor{gray!10}}c
|ccccc}
\toprule
\multirow{2}{*}{\textbf{Missing Rate}}
& \multicolumn{5}{c|}{\cellcolor{gray!10}\textbf{Incomplete data}}
& \multicolumn{5}{c}{\textbf{Complete data}} \\
& HyperImpute & PC-GAIN & GACN & 1-NN & IICN
& HyperImpute & PC-GAIN & GACN & 1-NN & IICN \\
\midrule
$0.1$
& 33.33 & 46.88 & 48.44 & 56.28 & 66.67
& 98.52 & 98.72 & 98.80 & 56.06 & 99.16 \\
\midrule
$0.3$
& 33.33 & 51.80 & 66.67 & 56.28 & 66.67
& 90.98 & 96.76 & 66.67 & 56.48 & 99.10 \\
\midrule
$0.5$
& 66.67 & 55.78 & 66.67 & 56.28 & 66.67
& 66.67 & 93.48 & 66.67 & 56.22 & 99.10 \\
\midrule
$0.7$
& 66.67 & 66.67 & 66.67 & 56.28 & 66.67
& 66.67 & 66.67 & 66.67 & 55.56 & 99.10 \\
\midrule
$0.9$
& 66.67 & 66.67 & 66.67 & 56.28 & 66.67
& 66.67 & 66.67 & 66.67 & 56.34 & 99.10 \\
\bottomrule
\end{tabular}}
\vspace{2mm}
\caption{Test accuracy (\%) on complete and incomplete samples
from the synthetic Burger dataset under different missing rates.
}
\label{tab:burger_more}
%\end{table}
% ==================== FashionMNIST ====================
%\begin{table}[t]
\centering
\setlength{\tabcolsep}{4pt}
\renewcommand{\arraystretch}{1.1}
\resizebox{\textwidth}{!}{
\begin{tabular}{cc|
>{\columncolor{gray!10}}c
>{\columncolor{gray!10}}c
>{\columncolor{gray!10}}c
>{\columncolor{gray!10}}c
|cccc}
\toprule
\multicolumn{2}{c|}{\textbf{Missing Rate}}
& \multicolumn{4}{c|}{\cellcolor{gray!10}\textbf{Incomplete data}}
& \multicolumn{4}{c}{\textbf{Complete data}} \\
$P_{\mathrm{obs}}$
& $P_{\mathrm{miss}}$
& PC-GAIN & GACN & 1-NN & IICN
& PC-GAIN & GACN & 1-NN & IICN \\
\midrule
\multirow{2}{*}{$0.1$}
& $0.5$
& $82.63\scriptstyle\pm.70$
& \cellcolor{gray!35}$\textbf{88.23}\scriptstyle\pm.60$
& $82.09\scriptstyle\pm.80$
& $\textbf{88.01}\scriptstyle\pm.60$
& $85.13\scriptstyle\pm.70$
& \cellcolor{gray!35}$\textbf{89.43}\scriptstyle\pm.60$
& $83.41\scriptstyle\pm.70$
& $\textbf{89.13}\scriptstyle\pm.60$ \\
& $0.8$
& $69.86\scriptstyle\pm.90$
& \cellcolor{gray!35}$\textbf{83.66}\scriptstyle\pm.70$
& $60.58\scriptstyle\pm1.00$
& $\textbf{83.06}\scriptstyle\pm.70$
& $76.88\scriptstyle\pm.80$
& \cellcolor{gray!35}$\textbf{87.06}\scriptstyle\pm.70$
& $80.52\scriptstyle\pm.80$
& $86.06\scriptstyle\pm.70$ \\
\midrule
\multirow{2}{*}{$0.5$}
& $0.5$
& $81.44\scriptstyle\pm.80$
& $87.92\scriptstyle\pm.60$
& $82.30\scriptstyle\pm.70$
& \cellcolor{gray!35}$\textbf{88.59}\scriptstyle\pm.60$
& $86.42\scriptstyle\pm.70$
& $\textbf{89.41}\scriptstyle\pm.60$
& $84.05\scriptstyle\pm.70$
& \cellcolor{gray!35}$\textbf{90.01}\scriptstyle\pm.60$ \\
& $0.8$
& $76.96\scriptstyle\pm.80$
& $\textbf{83.06}\scriptstyle\pm.70$
& $61.59\scriptstyle\pm1.00$
& \cellcolor{gray!35}$\textbf{83.69}\scriptstyle\pm.70$
& $85.03\scriptstyle\pm.70$
& $\textbf{88.66}\scriptstyle\pm.60$
& $80.17\scriptstyle\pm.80$
& \cellcolor{gray!35}$\textbf{89.23}\scriptstyle\pm.60$ \\
\midrule
\multirow{2}{*}{$0.9$}
& $0.5$
& $81.29\scriptstyle\pm.80$
& \cellcolor{gray!35}$\textbf{87.62}\scriptstyle\pm.60$
& $83.50\scriptstyle\pm.70$
& $\textbf{87.04}\scriptstyle\pm.70$
& $88.53\scriptstyle\pm.60$
& \cellcolor{gray!35}$\textbf{89.69}\scriptstyle\pm.60$
& $84.67\scriptstyle\pm.70$
& $\textbf{89.42}\scriptstyle\pm.60$ \\
& $0.8$
& $73.92\scriptstyle\pm.90$
& $\textbf{80.73}\scriptstyle\pm.80$
& $61.35\scriptstyle\pm1.00$
& \cellcolor{gray!35}$\textbf{81.30}\scriptstyle\pm.80$
& $88.18\scriptstyle\pm.60$
& \cellcolor{gray!35}$\textbf{89.74}\scriptstyle\pm.60$
& $82.32\scriptstyle\pm.70$
& $88.77\scriptstyle\pm.60$ \\
\bottomrule
\end{tabular}}
\vspace{2mm}
\caption{Test accuracy (\%), reported as mean $\pm$ 95\% CI,
on complete and incomplete FashionMNIST images under different
missing probabilities. Shaded: the highest mean in each setting
and evaluation group. Bold: the best result and results whose
means lie within its confidence interval.}
\label{tab:fashionmnist_more}
%\end{table}

% ==================== CIFAR-10 ====================
%\begin{table}[t]
\centering
\setlength{\tabcolsep}{4pt}
\renewcommand{\arraystretch}{1.1}
\resizebox{\textwidth}{!}{
\begin{tabular}{cc|
>{\columncolor{gray!10}}c
>{\columncolor{gray!10}}c
>{\columncolor{gray!10}}c
>{\columncolor{gray!10}}c
|cccc}
\toprule
\multicolumn{2}{c|}{\textbf{Missing Rate}}
& \multicolumn{4}{c|}{\cellcolor{gray!10}\textbf{Incomplete data}}
& \multicolumn{4}{c}{\textbf{Complete data}} \\
$P_{\mathrm{obs}}$
& $P_{\mathrm{miss}}$
& PC-GAIN & GACN & 1-NN & IICN
& PC-GAIN & GACN & 1-NN & IICN \\
\midrule
\multirow{2}{*}{$0.1$}
& $0.5$
& $62.65\scriptstyle\pm.90$
& $\textbf{75.48}\scriptstyle\pm.80$
& $35.32\scriptstyle\pm.90$
& \cellcolor{gray!35}$\textbf{75.92}\scriptstyle\pm1.10$
& $65.64\scriptstyle\pm.90$
& $77.36\scriptstyle\pm.80$
& $35.41\scriptstyle\pm.90$
& \cellcolor{gray!35}$\textbf{78.98}\scriptstyle\pm1.10$ \\
& $0.8$
& $44.02\scriptstyle\pm1.00$
& $66.45\scriptstyle\pm.90$
& $32.04\scriptstyle\pm.90$
& \cellcolor{gray!35}$\textbf{68.38}\scriptstyle\pm1.30$
& $54.69\scriptstyle\pm1.00$
& $71.95\scriptstyle\pm.90$
& $34.86\scriptstyle\pm.90$
& \cellcolor{gray!35}$\textbf{76.35}\scriptstyle\pm1.40$ \\
\midrule
\multirow{2}{*}{$0.5$}
& $0.5$
& $60.29\scriptstyle\pm1.00$
& \cellcolor{gray!35}$\textbf{75.12}\scriptstyle\pm.80$
& $34.90\scriptstyle\pm.90$
& $\textbf{74.77}\scriptstyle\pm.90$
& $71.13\scriptstyle\pm.90$
& \cellcolor{gray!35}$\textbf{78.64}\scriptstyle\pm.80$
& $34.96\scriptstyle\pm.90$
& $\textbf{78.27}\scriptstyle\pm.80$ \\
& $0.8$
& $44.18\scriptstyle\pm1.00$
& $65.15\scriptstyle\pm.90$
& $31.90\scriptstyle\pm.90$
& \cellcolor{gray!35}$\textbf{67.00}\scriptstyle\pm1.10$
& $64.95\scriptstyle\pm.90$
& $\textbf{75.57}\scriptstyle\pm.80$
& $34.59\scriptstyle\pm.90$
& \cellcolor{gray!35}$\textbf{76.00}\scriptstyle\pm1.20$ \\
\midrule
\multirow{2}{*}{$0.9$}
& $0.5$
& $48.36\scriptstyle\pm1.00$
& $71.42\scriptstyle\pm.90$
& $35.46\scriptstyle\pm.90$
& \cellcolor{gray!35}$\textbf{76.00}\scriptstyle\pm1.00$
& $67.96\scriptstyle\pm.90$
& \cellcolor{gray!35}$\textbf{78.50}\scriptstyle\pm.80$
& $35.46\scriptstyle\pm.90$
& $\textbf{78.48}\scriptstyle\pm1.10$ \\
& $0.8$
& $25.14\scriptstyle\pm.90$
& $51.75\scriptstyle\pm1.00$
& $33.45\scriptstyle\pm.90$
& \cellcolor{gray!35}$\textbf{53.84}\scriptstyle\pm1.14$
& $71.14\scriptstyle\pm.90$
& $74.90\scriptstyle\pm.80$
& $35.11\scriptstyle\pm.90$
& \cellcolor{gray!35}$\textbf{78.12}\scriptstyle\pm1.13$ \\
\bottomrule
\end{tabular}}
\vspace{2mm}
\caption{Test accuracy (\%), reported as mean $\pm$ 95\% CI,
on complete and incomplete CIFAR-10 images under different
missing probabilities. Shaded: the highest mean in each setting
and evaluation group. Bold: the best result and results whose
means lie within its confidence interval.}
\label{tab:cifar10_more}
%\end{table}
% Add these packages to the preamble:
% \usepackage{booktabs}
% \usepackage{multirow}
% \usepackage[table]{xcolor}

%\begin{table}[t]
\centering
\setlength{\tabcolsep}{4pt}
\renewcommand{\arraystretch}{1.1}
\resizebox{\textwidth}{!}{
\begin{tabular}{lccccccccc}
\toprule
\textbf{Dataset}
& \textbf{\#Feat.}
& \textbf{\#Samp.}
& \textbf{\#Comp.}
& \textbf{$P_{\mathrm{miss}}$}
& HyperImpute & PC-GAIN & GACN & 1-NN & IICN \\
\midrule

\multirow{2}{*}{Mammographic Mass}
& \multirow{2}{*}{5}
& \multirow{2}{*}{961}
& \multirow{2}{*}{831}
& \multirow{2}{*}{0.2473}
& $79.97\scriptstyle\pm0.179$
& $82.64\scriptstyle\pm0.172$
& $80.50\scriptstyle\pm0.216$
& $75.46\scriptstyle\pm0.21$
& \cellcolor{gray!35}$\textbf{83.15}\scriptstyle\pm0.149$ \\
& & & &
& \cellcolor{gray!10}$79.12\scriptstyle\pm1.417$
& \cellcolor{gray!10}$79.67\scriptstyle\pm1.652$
& \cellcolor{gray!10}$80.91\scriptstyle\pm1.776$
& \cellcolor{gray!10}$63.84\scriptstyle\pm0.72$
& \cellcolor{gray!35}$\textbf{82.86}\scriptstyle\pm1.173$ \\[2pt]

\midrule
\multirow{2}{*}{Credit Approval}
& \multirow{2}{*}{15}
& \multirow{2}{*}{690}
& \multirow{2}{*}{653}
& \multirow{2}{*}{0.1207}
& $75.81\scriptstyle\pm0.167$
& $85.92\scriptstyle\pm0.218$
& \cellcolor{gray!35}$\textbf{86.67}\scriptstyle\pm0.215$
& $78.53\scriptstyle\pm0.37$
& $85.92\scriptstyle\pm0.218$ \\
& & & &
& \cellcolor{gray!10}$\textbf{75.55}\scriptstyle\pm4.261$
& \cellcolor{gray!10}$\textbf{74.33}\scriptstyle\pm5.820$
& \cellcolor{gray!10}$\textbf{74.59}\scriptstyle\pm3.253$
& \cellcolor{gray!10}$\textbf{69.22}\scriptstyle\pm4.41$
& \cellcolor{gray!35}$\textbf{75.72}\scriptstyle\pm4.771$ \\[2pt]

\midrule
\multirow{2}{*}{Diabetes}
& \multirow{2}{*}{8}
& \multirow{2}{*}{768}
& \multirow{2}{*}{336}
& \multirow{2}{*}{0.2207}
& $75.55\scriptstyle\pm0.371$
& $77.03\scriptstyle\pm0.304$
& $67.02\scriptstyle\pm0.545$
& $67.13\scriptstyle\pm0.83$
& \cellcolor{gray!35}$\textbf{78.11}\scriptstyle\pm0.563$ \\
& & & &
& \cellcolor{gray!10}$63.22\scriptstyle\pm0.392$
& \cellcolor{gray!10}$74.12\scriptstyle\pm0.144$
& \cellcolor{gray!10}$63.46\scriptstyle\pm0.395$
& \cellcolor{gray!10}$66.01\scriptstyle\pm0.41$
& \cellcolor{gray!35}$\textbf{77.26}\scriptstyle\pm0.223$ \\[2pt]

\midrule
\multirow{2}{*}{Hepatitis}
& \multirow{2}{*}{19}
& \multirow{2}{*}{155}
& \multirow{2}{*}{80}
& \multirow{2}{*}{0.1001}
& $\textbf{76.78}\scriptstyle\pm3.169$
& $71.58\scriptstyle\pm0.745$
& $61.18\scriptstyle\pm2.302$
& $69.24\scriptstyle\pm2.31$
& \cellcolor{gray!35}$\textbf{77.98}\scriptstyle\pm3.397$ \\
& & & &
& \cellcolor{gray!10}$63.00\scriptstyle\pm1.070$
& \cellcolor{gray!10}$58.42\scriptstyle\pm1.228$
& \cellcolor{gray!10}$60.69\scriptstyle\pm1.678$
& \cellcolor{gray!10}$61.50\scriptstyle\pm1.42$
& \cellcolor{gray!35}$\textbf{68.36}\scriptstyle\pm2.656$ \\[2pt]

\midrule
\multirow{2}{*}{Thyroid}
& \multirow{2}{*}{21}
& \multirow{2}{*}{3772}
& \multirow{2}{*}{2642}
& \multirow{2}{*}{0.0966}
& $91.96\scriptstyle\pm0.015$
& $95.45\scriptstyle\pm0.015$
& $91.96\scriptstyle\pm0.018$
& $92.12\scriptstyle\pm0.021$
& \cellcolor{gray!35}$\textbf{95.53}\scriptstyle\pm0.013$ \\
& & & &
& \cellcolor{gray!10}$98.31\scriptstyle\pm0.032$
& \cellcolor{gray!10}$98.33\scriptstyle\pm0.039$
& \cellcolor{gray!10}$98.28\scriptstyle\pm0.026$
& \cellcolor{gray!10}$96.69\scriptstyle\pm0.081$
& \cellcolor{gray!35}$\textbf{98.38}\scriptstyle\pm0.012$ \\[2pt]

\bottomrule
\end{tabular}}
\vspace{2mm}
\caption{Test accuracy (\%), reported as mean $\pm$ standard deviation
over $5$-fold cross-validation on UCI real-world tabular datasets. For each dataset, the first
row reports performance on complete samples and the second row
reports performance on incomplete samples. Light gray indicates
incomplete-sample results; dark gray marks the highest mean in
each row. Bold indicates the best result and results whose means lie within one standard deviation of the best mean.}
\label{tab:uci_more}
\end{table}
\subsection{Empirical evaluation of pattern independence}\label{app:mask_independence}
We assess the dependence between completed images $\hat\mX$ and their missingness masks $\mM$ using the Hilbert-Schmidt Independence Criterion (HSIC~\citep{gretton2007kernel}). To reduce the dimensionality of the image representation, we fit a 100-component principal component analysis (PCA) transformation on complete training images and apply this fixed transformation to the completed test images. The masks retain their original dimensions. We compute HSIC using RBF kernels on the projected images and masks, and compare the observed statistic with a permutation baseline obtained by randomly shuffling whole masks across observations while keeping the completed images and PCA transformation fixed. We report the observed HSIC, and the permutation mean and standard deviation. This assesses dependence between the PCA image representation and the full mask.

The results are recorded in Table~\ref{tab:cifar_iicn_hsic},~\ref{tab:cifar_zero_hsic},~\ref{tab:fashion_iicn_hsic} and~\ref{tab:fashion_zero_hsic}. The HSIC diagnostics suggest that IICN exhibits behavior consistent with approximate mask independence in most evaluated settings. In the 100-dimensional PCA representation, the observed HSIC values are close to their permutation baselines across all CIFAR-10 settings and at $P_{\miss}=0.5$ on FashionMNIST. At $P_{\miss}=0.8$ on FashionMNIST, detectable slight dependence remains, although the deviations from the permutation baseline are substantially smaller than those of zero imputation. 
% ==================== IICN ====================
\begin{table}[h]
\centering
\setlength{\tabcolsep}{4pt}
\renewcommand{\arraystretch}{1.1}
\resizebox{\textwidth}{!}{
\begin{tabular}{cccccc}
\toprule
$P_{\mathrm{obs}}$ & $P_{\mathrm{miss}}$
& Observed HSIC
& Permutation mean $\pm$ SD
& Excess HSIC
& Standardized excess
 \\
\midrule
0.1 & 0.5 & 39.253172 & $39.300829 \pm 0.083094$ & $-0.047657$ & $-0.57$  \\
0.1 & 0.8 & 39.368867 & $39.279851 \pm 0.083302$ & 0.089016 & 1.07  \\
0.5 & 0.5 & 39.228040 & $39.294320 \pm 0.083284$ & $-0.066280$ & $-0.80$  \\
0.5 & 0.8 & 39.336494 & $39.259491 \pm 0.080458$ & 0.077003 & 0.96  \\
0.9 & 0.5 & 39.262573 & $39.311622 \pm 0.082654$ & $-0.049049$ & $-0.59$  \\
0.9 & 0.8 & 39.394871 & $39.284374 \pm 0.083078$ & 0.110497 & 1.33  \\
\bottomrule
\end{tabular}}
\vspace{2mm}
\caption{Mask-dependence diagnostics for IICN on CIFAR-10 using
100-dimensional PCA image representations and full 3072-dimensional
masks, with $10{,}000$ samples and $1{,}000$ permutations.
All HSIC values are in units of $10^{-6}$.
Excess HSIC is the observed statistic minus the permutation mean;
standardized excess divides this difference by the permutation SD.}
\label{tab:cifar_iicn_hsic}
%\end{table}
% ==================== Zero imputation ====================
%\begin{table}[t]
\centering
\setlength{\tabcolsep}{4pt}
\renewcommand{\arraystretch}{1.1}
\resizebox{\textwidth}{!}{
\begin{tabular}{cccccc}
\toprule
$P_{\mathrm{obs}}$ & $P_{\mathrm{miss}}$
& Observed HSIC
& Permutation mean $\pm$ SD
& Excess HSIC
& Standardized excess
 \\
\midrule
0.1 & 0.5 & 39.488830 & $39.381472 \pm 0.078845$ & 0.107366 & 1.36  \\
0.1 & 0.8 & 40.055141 & $39.415175 \pm 0.067338$ & 0.639966 & 9.50\\
0.5 & 0.5 & 39.488830 & $39.381472 \pm 0.078845$ & 0.107366 & 1.36 \\
0.5 & 0.8 & 40.055141 & $39.415175 \pm 0.067338$ & 0.639966 & 9.50  \\
0.9 & 0.5 & 39.488830 & $39.381472 \pm 0.078845$ & 0.107366 & 1.36  \\
0.9 & 0.8 & 40.055141 & $39.415175 \pm 0.067338$ & 0.639966 & 9.50  \\
\bottomrule
\end{tabular}}
\vspace{2mm}
\caption{Mask-dependence diagnostics for zero imputation on CIFAR-10
using the same PCA representation and evaluation settings as
Table~\ref{tab:cifar_iicn_hsic}.
All HSIC values are in units of $10^{-6}$.
Results are identical across $p_{\mathrm{obs}}$
for a fixed $p_{\mathrm{miss}}$ under the shared evaluation seeds.}
\label{tab:cifar_zero_hsic}
%\end{table}
% ==================== IICN ====================
%\begin{table}[t]
\centering
\setlength{\tabcolsep}{4pt}
\renewcommand{\arraystretch}{1.1}
\resizebox{\textwidth}{!}{
\begin{tabular}{cccccc}
\toprule
$P_{\mathrm{obs}}$ & $P_{\mathrm{miss}}$
& Observed HSIC
& Permutation mean $\pm$ SD
& Excess HSIC
& Standardized excess
 \\
\midrule
0.1 & 0.5 & 37.810716 & $37.747582 \pm 0.225835$ & 0.063134 & 0.28 \\
0.1 & 0.8 & 38.824160 & $37.863784 \pm 0.208453$ & 0.960376 & 4.61  \\
0.5 & 0.5 & 37.677353 & $37.815812 \pm 0.224386$ & $-0.138458$ & $-0.62$  \\
0.5 & 0.8 & 38.450945 & $37.816499 \pm 0.209609$ & 0.634446 & 3.03  \\
0.9 & 0.5 & 37.736094 & $37.909320 \pm 0.218881$ & $-0.173226$ & $-0.79$  \\
0.9 & 0.8 & 38.581105 & $37.831569 \pm 0.210080$ & 0.749536 & 3.57  \\
\bottomrule
\end{tabular}}
\vspace{2mm}
\caption{Mask-dependence diagnostics for IICN on FashionMNIST using
100-dimensional PCA image representations and full 784-dimensional
masks, with $10{,}000$ samples and $1{,}000$ permutations.
All HSIC values are in units of $10^{-6}$.
Excess HSIC is the observed statistic minus the permutation mean;
standardized excess divides this difference by the permutation SD.}
\label{tab:fashion_iicn_hsic}
%\end{table}

% ==================== Zero imputation ====================
%\begin{table}[t]
\centering
\setlength{\tabcolsep}{4pt}
\renewcommand{\arraystretch}{1.1}
\resizebox{\textwidth}{!}{
\begin{tabular}{cccccc}
\toprule
$P_{\mathrm{obs}}$ & $P_{\mathrm{miss}}$
& Observed HSIC
& Permutation mean $\pm$ SD
& Excess HSIC
& Standardized excess \\
\midrule
0.1 & 0.5 & 50.088523 & $38.726327 \pm 0.165774$ & 11.362196 & 68.54  \\
0.1 & 0.8 & 64.500073 & $39.289334 \pm 0.099110$ & 25.210738 & 254.37  \\
0.5 & 0.5 & 50.088523 & $38.726327 \pm 0.165774$ & 11.362196 & 68.54  \\
0.5 & 0.8 & 64.500073 & $39.289334 \pm 0.099110$ & 25.210738 & 254.37  \\
0.9 & 0.5 & 50.088523 & $38.726327 \pm 0.165774$ & 11.362196 & 68.54  \\
0.9 & 0.8 & 64.500073 & $39.289334 \pm 0.099110$ & 25.210738 & 254.37  \\
\bottomrule
\end{tabular}}
\vspace{2mm}
\caption{Mask-dependence diagnostics for zero imputation on
FashionMNIST using the same PCA representation and evaluation
settings as Table~\ref{tab:fashion_iicn_hsic}.
All HSIC values are in units of $10^{-6}$. Results are identical across $P_{\mathrm{obs}}$
for a fixed $P_{\mathrm{miss}}$ under the shared evaluation seeds.}
\label{tab:fashion_zero_hsic}
\end{table}
\subsection{Image reconstruction}
In this section, Table~\ref{tab:fashion_mse} and Table~\ref{tab:cifar_mse} present the reconstruction performance for FashionMNIST and CIFAR10 under different missing settings. For each cell, from top to the bottom is the original, the corrupted, the mean-imputed, the GAIN-imputed, the SupMIWAE-imputed, the LAC-imputed and the IICN-imputed. The MSE column records the mse per missing pixel. The calculation is $\mathrm{MSE}_{\mathrm{miss}}
=\frac{\|(\xx-G_{\boldsymbol{\theta}}(\tilde{\xx},\mm,\zz))\odot(1-\mm)\|_2^2}
{\|1-\mm\|_2^2}.$ We can observe that the high quality of imputation is not equal to superior classification performance.
\begin{table}[H]
%\begin{minipage}{\textwidth}
\centering
\textbf{FashionMNIST}
\renewcommand{\arraystretch}{2.35}
\resizebox{\textwidth}{!}{
\begin{tabular}{|p{0.325\textwidth}|c|p{0.325\textwidth}|c|p{0.325\textwidth}|c|}
\hline
\multicolumn{2}{|c|}{$P_{\obs}=0.9,\;P_{\miss}=0.5$} & \multicolumn{2}{c|}{$P_{\obs}=0.5,\;P_{\miss}=0.5$} & \multicolumn{2}{c|}{$P_{\obs}=0.1,\;P_{\miss}=0.5$}\\
\hline
\multirow{7}{*}{\includegraphics[width=\linewidth]{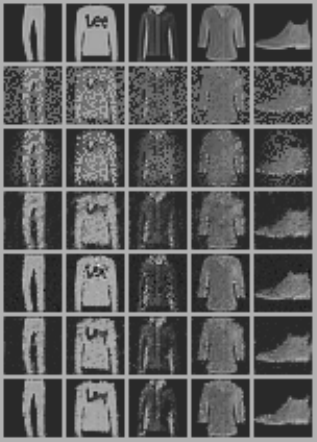}} & \textbf{MSE} &\multirow{6}{*}{\includegraphics[width=\linewidth]{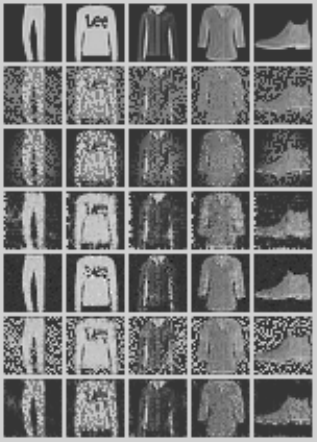}}& \textbf{MSE}&\multirow{6}{*}{\includegraphics[width=\linewidth]{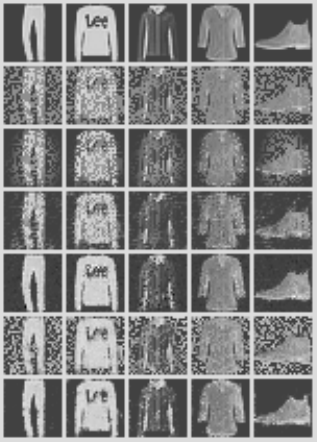}}& \textbf{MSE} \\\cline{2-2}\cline{4-4}\cline{6-6}
& 0.7058 & &0.7082 & &0.7057\\\cline{2-2}\cline{4-4}\cline{6-6}
& 0.3739 & &0.3521 & &0.3571\\\cline{2-2}\cline{4-4}\cline{6-6}
& 0.1509 & &0.3083 && 0.2071 \\\cline{2-2}\cline{4-4}\cline{6-6}
& 0.0999 & &0.0767 && 0.0895\\\cline{2-2}\cline{4-4}\cline{6-6}
& 0.1132& &1.1768 && 0.9733\\\cline{2-2}\cline{4-4}\cline{6-6}
& 0.0910& &0.2259 && 0.1120\\\cline{2-2}\cline{4-4}\cline{6-6}
\hline
\multicolumn{2}{|c|}{$P_{\obs}=0.9,\;P_{\miss}=0.8$} & \multicolumn{2}{c|}{$P_{\obs}=0.5,\;P_{\miss}=0.8$} & \multicolumn{2}{c|}{$P_{\obs}=0.1,\;P_{\miss}=0.8$}\\\hline
\multirow{7}{*}{\includegraphics[width=\linewidth]{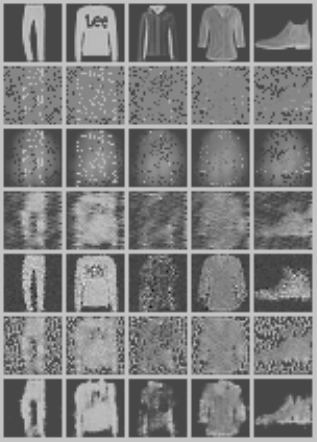}} & \textbf{MSE} &\multirow{6}{*}{\includegraphics[width=\linewidth]{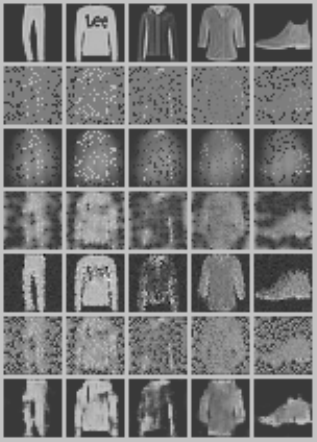}}&\textbf{MSE}&\multirow{6}{*}{\includegraphics[width=\linewidth]{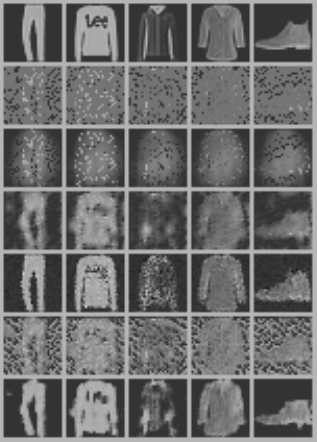}}& \textbf{MSE} \\\cline{2-2}\cline{4-4}\cline{6-6}
& 0.6992 & &0.7083 & &0.7046\\\cline{2-2}\cline{4-4}\cline{6-6}
& 0.3603 & &0.3611 & &0.3599\\\cline{2-2}\cline{4-4}\cline{6-6}
& 0.2913 & &0.3318 && 0.1977 \\\cline{2-2}\cline{4-4}\cline{6-6}
& 0.1496 & &0.1289 && 0.1544\\\cline{2-2}\cline{4-4}\cline{6-6}
& 0.6941& &0.6749 && 0.8061\\\cline{2-2}\cline{4-4}\cline{6-6}
& 0.1422& &0.1510 && 0.1505\\\cline{2-2}\cline{4-4}\cline{6-6}
\hline
\end{tabular}}
\vspace{1mm}
\caption{Qualitative reconstruction results and corresponding MSE values on FashionMNIST under different observation and missing probabilities. MSE denotes the mean squared reconstruction error averaged over missing pixels only. For each setting ($P_{\mathrm{obs}}, P_{\mathrm{miss}}$), we visualize representative samples together with their mean squared error (MSE), reported on the right of each image block. Lower MSE values generally correlate with visually improved reconstruction quality. This figure complements the main paper by providing additional qualitative comparisons.}
%\end{minipage}\par
\label{tab:fashion_mse}
\end{table}
\begin{table}[H]
%\vspace{0.5em}
%\begin{minipage}{\textwidth}
\centering
\textbf{CIFAR10}
\renewcommand{\arraystretch}{2.35}
\resizebox{\textwidth}{!}{
\begin{tabular}{|p{0.325\textwidth}|c|p{0.325\textwidth}|c|p{0.325\textwidth}|c|}
\hline
\multicolumn{2}{|c|}{$P_{\obs}=0.9,\;P_{\miss}=0.5$} & \multicolumn{2}{c|}{$P_{\obs}=0.5,\;P_{\miss}=0.5$} & \multicolumn{2}{c|}{$P_{\obs}=0.1,\;P_{\miss}=0.5$}\\
\hline
\multirow{7}{*}{\centering\includegraphics[width=\linewidth]{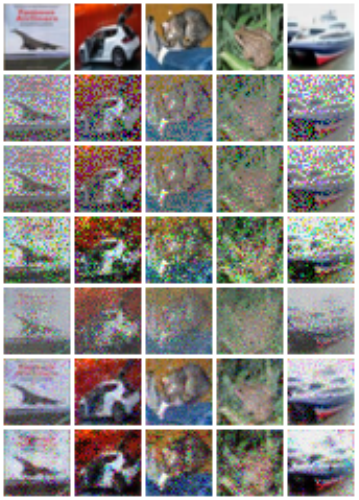}} & MSE &\multirow{6}{*}{\includegraphics[width=\linewidth]{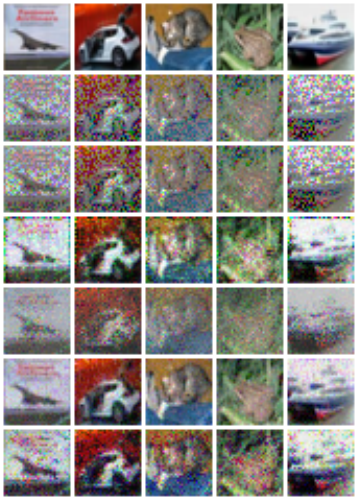}}& MSE&\multirow{6}{*}{\includegraphics[width=\linewidth]{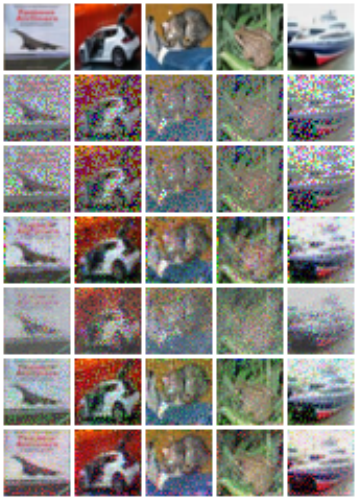}}& MSE \\\cline{2-2}\cline{4-4}\cline{6-6}
& 0.2808 & &0.2773 & &0.2750\\\cline{2-2}\cline{4-4}\cline{6-6}
& 0.2758 & &0.2715 & &0.2696\\\cline{2-2}\cline{4-4}\cline{6-6}
& 0.2294 & &0.1493 && 0.1510\\\cline{2-2}\cline{4-4}\cline{6-6}
& 0.1629 & &0.1581 && 0.1666\\\cline{2-2}\cline{4-4}\cline{6-6}
& 0.0623& &0.0383 && 0.0738\\\cline{2-2}\cline{4-4}\cline{6-6}
& 0.1128& &0.1307 && 0.0754\\\cline{2-2}\cline{4-4}\cline{6-6}
\hline
\multicolumn{2}{|c|}{$P_{\obs}=0.9,\;P_{\miss}=0.8$} & \multicolumn{2}{c|}{$P_{\obs}=0.5,\;P_{\miss}=0.8$} & \multicolumn{2}{c|}{$P_{\obs}=0.1,\;P_{\miss}=0.8$}\\\hline
\multirow{7}{*}{\includegraphics[width=\linewidth]{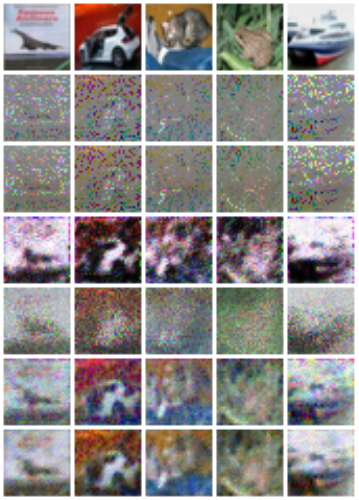}} & MSE &\multirow{6}{*}{\includegraphics[width=\linewidth]{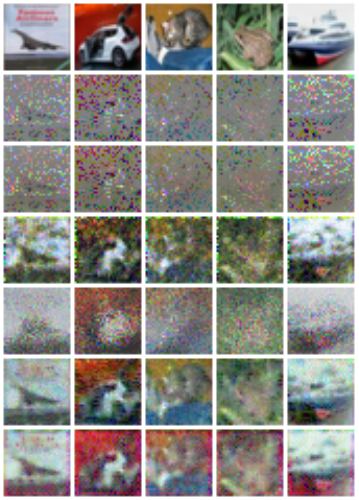}}& MSE&\multirow{6}{*}{\includegraphics[width=\linewidth]{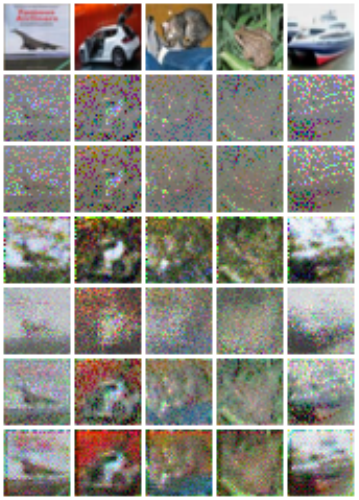}}& MSE \\\cline{2-2}\cline{4-4}\cline{6-6}
& 0.2767 & &0.2776 & &0.2748\\\cline{2-2}\cline{4-4}\cline{6-6}
& 0.2711 & &0.2718 & &0.2695\\\cline{2-2}\cline{4-4}\cline{6-6}
& 0.1440 & &0.1289 && 0.1285\\\cline{2-2}\cline{4-4}\cline{6-6}
& 0.1672 & &0.1749 && 0.1673\\\cline{2-2}\cline{4-4}\cline{6-6}
& 0.1045& &0.0654 && 0.1342\\\cline{2-2}\cline{4-4}\cline{6-6}
& 0.0567& &0.1054 && 0.1093\\\cline{2-2}\cline{4-4}\cline{6-6}
\hline
\end{tabular}}
%\end{minipage}
\vspace{1mm}
\caption{Qualitative reconstruction results and corresponding MSE values on CIFAR-10 under different observation and missing probabilities. MSE denotes the mean squared reconstruction error averaged over missing pixels only. For each setting ($P_{\mathrm{obs}}, P_{\mathrm{miss}}$), we visualize representative samples together with their mean squared error (MSE), reported on the right of each image block. Lower MSE values generally correlate with visually improved reconstruction quality. This figure complements the main paper by providing additional qualitative comparisons.}
\label{tab:cifar_mse}
\end{table}

\clearpage
\end{document}